\newif\ifarxiv

\arxivtrue

\ifarxiv
\documentclass[dvipsnames,runningheads,final]{llncs}
\else
\documentclass[dvipsnames,runningheads, final]{llncs}
\fi
\usepackage[T1]{fontenc}
\usepackage{graphicx}
\usepackage{hyperref}
\usepackage{hyperxmp}
\usepackage[obeyFinal]{todonotes}
\usepackage{amsmath}
\usepackage{cleveref}
\usepackage[shortlabels,inline]{enumitem}
\usepackage{catchfile}
\usepackage{titlecaps}
\usepackage{amssymb}
\usepackage{tikz}
\usepackage{mathrsfs}
\usetikzlibrary{shapes.geometric}
\usetikzlibrary{arrows,arrows.meta,calc}
\usetikzlibrary{shadows}
\usetikzlibrary{shapes.multipart}
\usetikzlibrary{positioning}
\usetikzlibrary{shadows}
\usetikzlibrary{shapes.multipart}
\usetikzlibrary{arrows.meta}
\usetikzlibrary{graphs,quotes}
\usepackage[ruled,linesnumbered,vlined,norelsize,noend]{algorithm2e}
\usepackage{relsize}
\usepackage{mathtools}
\usepackage{xcolor}
\usepackage{xspace}
\usepackage{amsfonts}
\usepackage[mathscr]{euscript}
\usepackage{stmaryrd}
\ifarxiv
\usepackage[bibliography=common]{apxproof}
\else
\usepackage[bibliography=common,appendix=strip]{apxproof}
\fi
\usepackage{booktabs}
\usepackage{csquotes}
\usepackage[misc,geometry]{ifsym}

\makeatletter
\patchcmd{\@algocf@start}
  {-1.5em}
  {1em}
{}{}
\makeatother

\crefname{algocf}{alg.}{algs.}
\Crefname{algocf}{Alg.}{Algs.}

\newcommand{\AYT}[1]{}
\newcommand{\BG}[1]{}
\newcommand{\CO}[1]{}
\newcommand{\MO}[1]{}

\ifarxiv
\newcommand{\commentMessage}[1]{}
\else
\newcommand{\commentMessage}[1]{#1}
\fi

\newcommand{\bente}[1]{}
\newcommand{\anni}[1] {}
\newcommand{\magdalena}[1]{}
\newcommand{\cem}[1]{}

\newcommand{\bg}[1]{#1}
\newcommand{\mo}[1]{#1}
\newcommand{\co}[1]{#1}
\newcommand{\ayt}[1]{#1}

\newcommand{\roleStarOperation}[1]{#1^*}
\newcommand{\roleStarSymbol}[1]{#1_\star\xspace}

\newcommand{\depth}{\mathit{depth}}

\newcommand{\Sig}{\ensuremath{\Sigma}\xspace}

\newcommand{\SigC}{\ensuremath{\Sig_{\mathsf{C}}}\xspace}
\newcommand{\SigCOf}[1]{\ensuremath{\SigC(#1)}\xspace}
\newcommand{\SigR}{\ensuremath{\Sig_{\mathsf{R}}}\xspace}
\newcommand{\SigROf}[1]{\ensuremath{\SigR(#1)}\xspace}

\newcommand{\exSigStar}{\ensuremath{\SigR^{\star}\xspace}}

\newcommand{\exSig}{\ensuremath{\SigR^{\star}\xspace}}

\newcommand{\NSetC}{\ensuremath{\mathsf{N_C}}\xspace}
\newcommand{\NSetR}{\ensuremath{\mathsf{N_R}}\xspace}
\newcommand{\NSetI}{\ensuremath{\mathsf{N_I}}\xspace}
\newcommand{\NSetS}{\ensuremath{\mathsf{N_S}}\xspace}

\newcommand{\NSetRMinus}{\ensuremath{\mathsf{N^\pm_R}}\xspace}

\newcommand{\NSetRStar}{\ensuremath{\mathsf{N_R^\star}}\xspace}

\newcommand{\Cmc}{\ensuremath{\mathcal{C}}\xspace}

\newcommand{\Fmc}{\ensuremath{\mathcal{F}}\xspace}
\newcommand{\Gmc}{\ensuremath{\mathcal{G}}\xspace}
\newcommand{\Imc}{\ensuremath{\mathcal{I}}\xspace}

\newcommand{\Lmc}{\ensuremath{\mathcal{L}}\xspace}
\newcommand{\Tmc}{\ensuremath{\mathcal{T}}\xspace}

\newcommand{\Umc}{\ensuremath{\mathcal{U}}\xspace}

\newcommand{\EL}{\ensuremath{\mathcal{E\!L}}\xspace}

\newcommand{\ELstar}{\ensuremath{\mathcal{E\!L}^*}\xspace}

\newcommand{\ELI}{\ensuremath{\mathcal{E\!LI}}\xspace}
\newcommand{\ELIstar}{\ensuremath{\mathcal{E\!LI}^*}\xspace}

\newcommand{\ourDL}{\ELIstar}

\newcommand{\graph}{\ensuremath{\mathcal{G}}\xspace}

\newcommand{\cat}{\mathcal{C}\xspace}
\newcommand{\decl}[2]{#1 
\coloneqq  #2}
\newcommand{\automaton}{\mathfrak{A}\xspace}
\newcommand{\run}{\ensuremath{\mathsf{r}}\xspace}

\newcommand{\PEx}{\ensuremath{\mathscr{P}}\xspace}  
\newcommand{\NEx}{\ensuremath{\mathscr{N}}\xspace} 

\newcommand{\fit}{\mathcal{F}\xspace}
\newcommand{\fittinginst}{\ensuremath{\fit=(\graph, \Cmc, \PEx, \NEx, \Sigma)}\xspace}
\newcommand{\fittingNoCat}{\ensuremath{\fit=(\graph, \PEx, \NEx, \Sigma)}\xspace}
\newcommand{\fittingNoCatSig}{\ensuremath{\fit=(\graph, \PEx, \NEx)}\xspace}

 \newcommand{\Unr}{\ensuremath{\Umc}\xspace}

 \newcommand{\FitProdStar}{\ensuremath{\Pi_\star(\Gmc, {\vec{v}_+}, \Sigma)}\xspace}
 \newcommand{\vecV}{\ensuremath{{\vec{v}_+}}\xspace}

\newcommand{\preceqSig}{\ensuremath{\preceq_{\Sigma}}\xspace}

 \newcommand{\ELIpreceq}
 {\ensuremath{\preceq^{\ELI}}\xspace}

\newcommand{\labelLang}{\mathbb{L}_{\Sigma}\xspace}

\newcommand{\Sem}{\ensuremath{\mathbb{S}}\xspace} 
\newcommand{\SemGra}{\ensuremath{\Sem({\graph,\cat})}\xspace}
\newcommand{\wfsem}{\ensuremath{\SemGra^{\mathsf{wf}\!}}\xspace}
\newcommand{\supsem}{\ensuremath{\SemGra^{\mathsf{sup}\!}}\xspace}
\newcommand{\stsem}{\ensuremath{\SemGra^{\mathsf{st}\!}}\xspace}

\newcommand{\wfass}{\ensuremath{\alpha^{\mathsf{wf}}_{\graph,\cat}}\xspace}

\newcommand{\semPath}[1]{\llbracket{#1}\rrbracket}
\newcommand{\semdelta}[3]{\llbracket{#1}\rrbracket^{#2}_{#3}}

\newtheoremrep{corollary}{Corollary}
\newtheoremrep{definition}{Definition}
\newtheoremrep{example}{Example}
\newtheoremrep{proposition}{Proposition}
\newtheoremrep{lemma}{Lemma}
\newtheoremrep{theorem}{Theorem}

\newcommand{\ourDLTree}{\mathcal{T}}
\newcommand{\nodeLabelTree}{\mathfrak{T}_C}
\newcommand{\nodeLabelTreeF}[2]{\mathsf{nlt}_{#1}({#2})}
\newcommand{\nodeLabelTreeWithoutC}{\mathfrak{T}}

\newcommand{\nodeLabelTreeOther}{\mathfrak{T}_{D}}

\newcommand{\starclosed}{$\star$-closed\xspace}
\newcommand{\rsr}[1]{\ensuremath{\mathsf{rsr}({#1})}\xspace}

\newcommand*\circled[1]{\tikz[baseline=(char.base)]{
            \node[shape=circle,draw, thick, inner sep=0pt] (char) {#1};}}

\newcommand{\nodebullet}{\LARGE $\bullet$}
\newcommand{\nodepositive}{\color{PineGreen} \circled{\tiny $+$}}

\newcommand{\nodenegative}{\color{Maroon} \circled{\tiny $-$}}

\newcommand{\nodepositiveText}{\raisebox{2pt}{\nodepositive}\xspace}
\newcommand{\nodenegativeText}{\raisebox{2pt}{\nodenegative}\xspace}

\newcommand{\OMIT}[1]{}

\newcommand{\pathExpr}{\pi}

\newcommand{\eq}{\mathsf{eq}}
\newcommand{\disj}{\mathsf{disj}}
\newcommand{\hasvalue}{\mathsf{test}}
\newcommand{\test}{\mathsf{test}}
\newcommand{\closed}{\mathsf{closed}}

\newcommand{\id}{\mathsf{id}}

\newcommand{\geqn}[2]{\exists^{\geq #1}#2.}
\newcommand{\leqn}[2]{\exists^{\leq #1}#2.}

\newcommand{\iexpr}[2]{\llbracket #1 \rrbracket_{#2}}

\usepackage{bbm}

\newcommand{\gDef}{{ \ \rightarrow \ }}
\newcommand{\gMid}{{\ \big|\ }}
\newcommand{\gEnd}{}

\begin{document}
\title{
Shapes from Examples: Foundations of Shape Learning in Recursive SHACL
}
%
%
\author{Bente Gortworst \inst{1}\,\Letter\,\orcidID{0009-0007-9880-1592} \and \hfill
Cem Okulmus \inst{2}\,\orcidID{0000-0002-7742-0439} \and
Magdalena Ortiz \inst{1}\,\orcidID{0000-0002-2344-9658} \and \\
Anni-Yasmin Turhan \inst{2}\,\orcidID{0000-0001-6336-335X}}
\authorrunning{B. Gortworst et al.}
%
\institute{TU Wien, Vienna, Austria \\ \email{bente.gortworst@tuwien.ac.at, ortiz@kr.tuwien.ac.at}\\
\and
Paderborn University, Paderborn, Germany \\
\email{cem.okulmus@upb.de, turhan@upb.de}
}
\maketitle              
\begin{abstract}
SHACL shapes enable data graph validation, making automatic shape learning essential for knowledge graph applications.
We investigate the well-known 
\emph{fitting}  approach to this task:
given sets \PEx and \NEx  of positive and negative example nodes from an input graph, compute a shape expression $C$, possibly using shape names defined in a recursive shape catalogue, that validates at every node in \PEx and at none in \NEx. 
 We focus on the case where $C$ is written in a core fragment of SHACL corresponding to the Description Logic $\ELI^*$. For the catalogue, we consider the well-founded, stable, and supported semantics. We address fitting existence and most specific fitting computation, establish tight exponential-time upper bounds for both problems, and obtain polynomial bounds for relevant special cases.

\end{abstract}

\section{Introduction}


Knowledge graphs and RDF graphs are increasingly employed in practical applications, but unlike relational DBs, do not impose rigid schemas. This raises the need to ensure data quality by other means such as \emph{graph validation}, which essentially tests whether some structural constraints are fulfilled by the data. 
Two prominent languages for graph validation are the W3C recommendation \emph{Shapes Constraint Language}~\cite{KK17} (SHACL) and the \emph{Shape Expressions language} (ShEx)~\cite{PGS14}. SHACL expresses constraints on RDF graphs by defining complex \emph{shapes} that can be validated at target nodes of a graph. A SHACL constraint is a set of pairs of a condition on the target nodes and a shape declaration, which is, in turn, a pair of shape name and expression i.e.\ the actual constraint on the graph. 
The SHACL semantics is given by assignments that map declared shape names to graph nodes in a way compliant with the expression given in the constraint.
Validation is to  test whether some given targets are made true by the 
 assignment.
For full SHACL with cyclic shape dependencies and full negation,
there are some competing proposals for the semantics~\cite{ACORSS20,BJ21,CRS18,DBLP:conf/kr/OkulmusS24}.

Vast amounts of RDF data evolved into large and messy graphs long before SHACL was introduced. Extracting shapes from these graphs is one of the most promising paths towards understanding their structure and making them available for validation, querying and reuse.
Thus learning SHACL shapes from data examples is vital. 
\bg{Similar tasks have been investigated using \ayt{statistical} machine learning-based methods in \cite{fernandez2018inference,Mihindukulasooriya2018RDFSI,DBLP:journals/semweb/OmranTMH23}, where a predicate or a set of nodes are taken as input and are embedded into a space that can be learned by neural networks. Furthermore, Fernández-Álvarez et al.~\cite{DBLP:journals/kbs/Fernandez-Alvarez22} explored the use of combinatorial methods to manipulate a set of target shapes so that they better match a given graph.}
 \ayt{These works are using pragmatic algorithms for learning shapes from input}. 
 \AYT{Bente, I found your version too polite and vague.}
 So far systematic attempts at understanding the \bg{formal} properties and limits of such learning algorithms \co{for SHACL} are,
  \bg{to \co{the best of} our knowledge, \ayt{missing}}. \bg{Hence, we provide SHACL shape learning algorithms with  formal guarantees \ayt{in terms} of soundness, completeness and termination. Furthermore, we \ayt{analyse} the complexity of the proposed methods. In doing so we lay foundations for \ayt{the development of} efficient learning algorithms and \ayt{extensions of our methods} to broader fragments of SHACL.} 

A well-investigated approach to learning formulas from positive and negative examples is that of computing so-called \emph{fittings}, i.e., complex formulae 
that capture the commonalities of the positive examples and separate them from the negative ones. This approach has been studied for many settings, e.g., learning conjunctive queries (CQs) from data bases \cite{cate2025extremalfittingproblemsconjunctive}, learning LTL formulas from temporal interpretations \cite{TT-SAC-22,DBLP:conf/ijcai/JungRWZ24} 
learning description logic (DL) concepts from interpretations \cite{CKO-IJCAI-24,Distel2011} or ABoxes together with TBoxes \cite{DBLP:conf/ijcai/FunkJLPW19,DBLP:conf/ijcai/ZarriessT13}.

Given a graph together with positive and negative example nodes, computing a \emph{SHACL fitting} is to find a shape that validates at every positive node but at none of the negative nodes. We study this problem under recursive shape catalogues and the well-founded, stable, and supported semantics of SHACL.
We allow to restrict the signature of the fitting shape. For instance, we may want to allow only some of the data predicates, or to disallow some defined shape names. 
Full expressivity of SHACL shapes renders the fitting problem uninteresting (as disjunction and nominals admit lists of positive examples as trivial shapes) 
hence we chose a SHACL fragment that corresponds to the DL \ourDL. It can express (using $r^*$) the existence of paths over the reflexive transitive closure of possibly inverse relations. 
\ourDL  captures tree CQs and extends them with reachability, which is crucial for describing graph shaped data. 
%
%
\vspace{-\smallskipamount}
\begin{example}[Product Order History]\label{exampleintro-withshapes}
\bg{The graph below represents a company's workflow for a product order and shipping process, where} edges $r,s$ and $t$ mark processes, $O$ stands for an incoming order, $S$  for \emph{shipping} and $A$  for \emph{abort}. \ayt{The company's shape catalogue $\cat$  defines shape $B$, stating how \emph{abort} could be reached.} To find \emph{anomalies}, \nodepositiveText represents unusual and \nodenegativeText typical cases.  The user wants an expression that separates unusual cases from typical cases. 
%
%
%
%
%
%
%
%
\\ \noindent
\parbox[c]{0.655\textwidth}
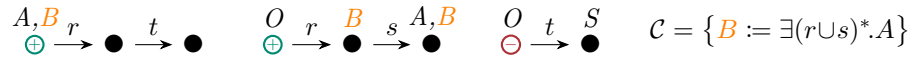
{
	\centering
	\begin{tikzpicture}[scale=0.7,->, >=Stealth, node distance=2cm, every node/.style = {font=\normalsize}]
		
		\node [label={[label distance=-2mm]90:A,{\color{orange}B}}] (a) at (0,0) {\nodepositive};
		\node  (a2) at (1.5,0) { \nodebullet};
		\node  (a3) at (3,0) { \nodebullet};
		
		\node [label={[label distance=-1.5mm]90:O}] (b) at (4.5,0) {\nodepositive};
		\node [label={[label distance=-1.5mm]90:{\color{orange}B}}] (b2) at (6,0) { \nodebullet};
		\node [label={[label distance=-2mm]90:A,{\color{orange}B}}] (b3) at (7.5,0) { \nodebullet};

		\node  [label={[label distance=-1.5mm]90:O}] (c) at (9,0) {\nodenegative};
		\node  [label={[label distance=-1.5mm]90:S}] (c2) at (10.5,0) {\nodebullet};

		\path (a) edge node[above] {$r$} (a2);
		\path (a2) edge node[above] {$t$} (a3);
		\path (b) edge node[above] {$\bg{r}$} (b2);
		\path (b2) edge node[above] {$s$} (b3);
		\path (c) edge node[above] {$t$} (c2);
	\end{tikzpicture}
	} 
	\hfill
	\parbox[c]{0.31\textwidth}{
		$\cat ={} \big\{ \decl{{\color{orange}B}}{\exists(r^*\cup s^*)\!.A  } \big\}\;$
	}

\noindent
An \ourDL fitting that separates the unusual cases from the typical one, 
is \bg{$\exists r^*.\color{orange}B$}.
\end{example}
The (reflexive) transitive closure has 
hardly been explored for DL fittings yet \cite{BTK-LPAR-03}. 
Interestingly, $r^*$-roles can express disjunctions over $r$-paths of differing length,  adding expressivity. 
\bg{To see this, consider again \Cref{exampleintro-withshapes} and note that the only fitting \ayt{expressible} in the DL \ELI (which does not have $r^*$) is the shape $\exists r.\top$ which is less informative than fittings expressed in \ourDL. }


Under the well-founded semantics (WFS) and without negation of defined shapes in the catalogue, our setting is 
analogous
to DL terminologies under least fixpoint semantics. Not much is known regarding the computation of fittings in that setting. \bg{In fact, seemingly the only work investigating learning in DLs  under fixpoint semantics is \cite{10.5555/1630659.1630707}. It is thus the closest reference for  the learning problems we study in this paper.} 

As fittings need not exist for cyclic graphs, we investigate 
deciding existence of fittings (w.r.t.\ restricted signatures). 
For \ELI concepts, this problem is known to be 
\textsc{ExpTime}-complete \cite{cate2025extremalfittingproblemsconjunctive}. We show that the increased expressivity of $r^*$-paths does not change the complexity. To get this result we adapted the approach based on two-way alternating parity tree automata from \cite{cate2025extremalfittingproblemsconjunctive}  to $r^*$-paths.

\ayt{Sometimes, there may be many solutions to a fitting problem. In such a case, we might prefer fittings that give more specific information than all the others. These are so-called \emph{most specific fittings (MSFs)}.} We study MSFs expressed in \ourDL and give procedures for deciding existence and even computing the MSF, if it exists. 
The standard approach for this is to construct the cross-product of the graphs reachable from positive examples and testing if the resulting product excludes the negative examples. Our procedure  for \ourDL adapts this by first extending the graphs by new $r^*$-edges for each $r$, in order to encode the reflexive transitive closure of the $r$-reachability, and then adapting the notion of simulations 
\AYT{Simulations come in here a little unexpectedly.}
to detect if the product of these new graphs is a fitting. 
We present an algorithm to compute the MSF and show that under a bounded number of positive examples and for tractable validation (i.e., for non-recursive SHACL and for recursive SHACL under  WFS), we can decide fitting existence and compute a representation of the MSF even in polynomial time.

\ifarxiv
\noindent
The proofs for the results attained in this paper are in the appendix.
\fi
%

\section{Preliminaries} 
\label{sec:preliminaries}

We consider disjoint, countably infinite sets of \emph{class names} $\NSetC$, \emph{role names} $\NSetR$, nodes $\NSetI$ and of \emph{shape names} $\NSetS \subseteq \NSetC$ and let $\NSetRMinus = \{t, t^- \! \mid t \in \NSetR \}$.
%
A \emph{data graph}  $\graph = (V,E,\ell)$ is a labelled graph, where $V \subseteq \NSetI$ is the \emph{finite set of nodes}, $E \subseteq \NSetI \times \NSetR \times \NSetI $ is the \emph{set of labelled edges} and $\ell: V \rightarrow 2^{\NSetC}$ is the node-labelling function, assigning to each node a set of classes. For 
$r\in\NSetRMinus$,
we 
write $r(a,b) \in E$
whenever $(a,r,b) \in E$ or $(b,r^-,a) \in E$. 
A \emph{pointed graph} is  a pair $(\graph, a)$ of a data graph $\graph = (V,E,\ell)$ and a \emph{distinguished} node $a\in V$.
We sometimes differentiate between \emph{adorned} data graphs, that allow for shape names in the labelling function $\ell$, and \emph{pure} data graphs with $\ell: V \rightarrow (2^{\NSetC \setminus \NSetS})$.
We give a formal definition of SHACL, 
in line with other recent works~\cite{DBLP:conf/www/AhmetajBHHJGMMM25,MJPHD}.
%
\begin{definition}
  \label{def:shacl-shape}
  Let   $X \in \NSetC \cup \NSetI$, $Q \subseteq \NSetR$, $S \in \NSetS$ and $r \in \NSetRMinus$.
  A \emph{SHACL shape expression} (or  \emph{shape}) $\varphi$ 
  and a \emph{path expression $\pathExpr$} obey the following grammar:
  { \small
  \begin{align*}
    \varphi
  \gDef \
  & \top
  \gMid S
  \gMid \hasvalue(X)  
  \gMid \closed(Q) 
  \gMid \eq(\pathExpr, r) \gMid \disj(\pathExpr, r)
    \gMid \neg \varphi \gMid \varphi \land \varphi
    \gMid \geqn{n}{\pathExpr}{\varphi} \\
 \pi \gDef \ & \id 
  \gMid r
  \gMid \pathExpr^{-}
  \gMid \pathExpr \cdot \pathExpr
  \gMid \pathExpr \cup \pathExpr
  \gMid (\pathExpr)^{*} \ .
 \gEnd
  \end{align*}
  }
   \end{definition}
%
%

\bg{For a SHACL shape $\varphi$ whose grammar is restricted to some fragment $\Lmc$ of full SHACL, we write that $\varphi$ is an $\Lmc$ shape.}
On a graph $\graph = (V,E,\ell)$, 
$\pi$ defines a binary relation 
$\semPath{\pi}_{\graph}$,
see \Cref{tab:seme2}. 
For the semantics of shapes, we use a \emph{shape assignment}, a partial function $\alpha: \NSetS \rightarrow 2^{\NSetI}$.
  %
The \emph{evaluation of $\varphi$ under $\alpha$ in \graph},
  denoted $\semdelta{\varphi}{\alpha}{\graph}$, is as in \Cref{tab:shape-evaluation}. 
  We may  write $\semdelta{\varphi}{}{\graph}$ if $\alpha$ is empty.
A \emph{shape declaration} is $\decl{S}{\varphi}$, 
with  $S$ a shape name and $\varphi$ a shape, and a \emph{shape catalogue} $\cat$ is a finite set of shape declarations.
%
%
\begin{table}[t]
	\begin{minipage}[t]{0.43\textwidth}
		\small
		\caption{Path expression evaluation}
		\vspace{-2mm}
		\label{tab:seme2}
		\setlength{\tabcolsep}{-1pt}
		\begin{tabular}{p{1cm}l}
			\toprule
			$\pathExpr$ & $\iexpr{\pathExpr}{\graph} \subseteq \NSetI \times \NSetI$, $\graph = (V,E,\ell)$  \\ 
			\midrule
			$\id$ & $\{ (v,v) \mid v \in \NSetI \}$\\[2pt]
			$r$ & $\{(v,u)\mid r(v,u)\in E \}$ \\[2pt]
			$\pathExpr^{-}$ & $\{(v,u)\mid (u,v) \in \iexpr{\pathExpr}{\graph}\}$  \\[2pt]
			$\pathExpr \cdot \pathExpr'$ & $\{(v,u) \mid \exists v':$  \\[1.5pt] 
			&$\phantom{\{ }(v,v')\in\iexpr{\pathExpr}{\graph}, (v',u)\in\iexpr{\pathExpr'}{\graph}\}$ \\[2pt]
			$\pathExpr\cup \pathExpr'$ & $\iexpr{\pathExpr}{\graph}\cup\iexpr{\pathExpr'}{\graph}$\\[2pt]
			$(\pathExpr)^{*}$ & $ \iexpr{\id}{\graph} \cup \iexpr{\pathExpr}{\graph}  \cup \iexpr{\pathExpr \cdot \pathExpr}{\graph} \cup \ldots $ \\[2pt]
			\bottomrule
		\end{tabular}
	\end{minipage} 
	\begin{minipage}[t]{0.42\textwidth}
		\setlength{\tabcolsep}{-1pt}
		\caption{Shape evaluation}
		\vspace{-2mm}
		\label{tab:shape-evaluation}
		\small
		\begin{tabular}{p{1.5cm}l}
			\toprule
			$\varphi$ & $\semdelta{\varphi}{\alpha}{\graph} \subseteq \NSetI$, $\graph = (V,E,\ell)$ \\
			\midrule
			$\top$ & $\NSetI$ \\
			$\test(X)$ & $\{ v \in V \mid X \in \ell(v) \} \cup \{ v \in V \mid v = X \}$ \\   
			$S$ & $ \alpha(S) \cup \{ v \in V \mid S \in \ell(v) \} $\\
			$\neg \varphi$ & $ \NSetI \setminus \semdelta{\varphi}{\alpha}{\graph}$\\
			$\varphi_1 \land \varphi_2$ & $\semdelta{\varphi_1}{\alpha}{\graph} \cap \semdelta{\varphi_2}{\alpha}{\graph}$\\
			$\geqn{n}{\pi}{\varphi}$  & 
			$\big\{v \in V \mid\! \#\big\{ u \in \semPath{\pi}_{\graph}^v\! \mid\! u \in \semdelta{\varphi}{\alpha}{\graph}\big\} \ge n \big\}$ \\
			$ \closed(Q) $ & 
			$\big\{ v \in V \mid 
			\semPath{r}_\graph^v = \emptyset\;\text{for all}\; r\in \NSetR \setminus Q\big\} $\\
			$\eq(\pi, r)$ &  
			$\big\{v \in V \mid \semPath{\pi}_\graph^v = \semPath{r}_\graph^v \big\} $\\ 
			$\disj(\pi, r)$ &
			$ \big\{ v \in V \mid \semPath{\pi}_\graph^v \cap \semPath{r}_\graph^v  = \emptyset \big\}$    \\
			\bottomrule
		\end{tabular}
		(where $\semPath{\pi}_{\graph}^v = \{u \mid (v, u) \in \semPath{\pi}_{\graph}\}$)
	\end{minipage}
\end{table}

We sometimes consider 
\emph{$\Sigma$-shapes} over a restricted signature 
$\Sigma \subseteq {\NSetR} \cup \NSetC$, as we define
$\SigC = \Sig \cap \NSetC$ and 
$\SigR = \Sig \cap {\NSetR}$.
%
%
By $\Sigma(X)$ we denote the role and class names that occur in
a graph, shape, or shape catalogue $X$, and we may use
%
$\SigCOf{X}$ and $\SigROf{X}$, with the expected meaning.  

%
\begin{definition}
A \emph{SHACL semantics}  \SemGra is a function $\Sem$ mapping each pair of a data graph $\graph$ with nodes $V$ and a shape catalogue $\cat$ to a set of shape assignments $\alpha: \SigCOf{\cat} \cap \NSetS \rightarrow 2^{V \cup \NSetI(\cat)}$, where $\NSetI(\cat)$ is the set of individuals mentioned in $\cat$. 
\end{definition}
Note that 
if $\cat$ is empty, \SemGra
can only contain the empty shape assignment. 
We consider recursive SHACL catalogues, with unrestricted cyclic dependencies and negation, for which 
there are three well-known semantics: 
the supported semantics (SUS)~\cite{CRS18}, the stable semantics (STS)~\cite{ACORSS20} and the well-founded semantics (WFS)~\cite{DBLP:conf/kr/OkulmusS24}. 
\ifarxiv
The full definitions are in the appendix. 
\fi
In a nutshell, assuming a single 
declaration $\decl{S}{\varphi}$ for each shape name $S$, 
the SUS accepts all shape assignments with  $\alpha(S) = \semdelta{\varphi}{\alpha}{\graph}$, while the STS requires the assignment to satisfy some minimality conditions; intuitively, every  node in a shape assignment must be `justified'. The WFS can be seen as a tractable approximation of the STS that assigns a node to a shape if it is assigned in every stable assignment. 
If there are no cyclic definitions involving negated shape names, there is a unique stable assignment that coincides with the well-founded one \cite{ACORSS20}.

\nosectionappendix
\begin{toappendix}

\section{Semantics of Recursive SHACL Evaluation} 
We will provide in this section the definitions for the semantics, as functions from data graph $\graph$ and shape catalogues $\cat$ to sets of shape assignments.

\newcommand{\ptrue}[3]{\lfloor #1 \rfloor^{#3}_{#2}}
\newcommand{\mtrue}[3]{\lceil #1 \rceil^{#3}_{#2}}
\newcommand{\pmtrue}[3]{[ #1 ]^{#3}_{#2}}
 
\newcommand{\maxunfounded}{\mathcal{U}}

\newcommand{\conseq}[2]{T_{#1}(#2)}

\newcommand{\head}{\mathsf{head}}

\newcommand{\body}{\mathsf{body}}

\newcommand{\naf}{\mathit{not}~\xspace}

\paragraph{Well-Founded Semantics.}

\begin{figure*}[t]\centering
  \begin{align*}        
\pmtrue{\top}{\graph}{\Gamma} &= \NSetI  &  
\pmtrue{\neg \top}{\graph}{\Gamma} &=   \emptyset          \\[.5ex]
    \ptrue{S}{\graph}{\Gamma} &= \{v\mid S(v) \in \Gamma\}  &    
    \mtrue{S}{\graph}{\Gamma} &= \{v\in V \mid  \neg S(v) \not\in \Gamma\} \\[.5ex]    
    \ptrue{\neg S}{\graph}{\Gamma} &= \{v \mid \neg S(v) \in \Gamma\} &
    \mtrue{\neg S}{\graph}{\Gamma} &= \{v\in V \mid S(v)\not \in \Gamma\} \\[.5ex]
         \pmtrue{\test(C)}{\graph}{\Gamma} &= \{ v \in V \mid C \in \ell(v) \} &  
\pmtrue{\neg \test(C)}{\graph}{\Gamma} &=  \{ v \in V \mid C \not \in \ell(v) \}          \\[.5ex]
     \pmtrue{\test(v)}{\graph}{\Gamma} &= \{ v \}  &  
     \pmtrue{\neg \test(v)}{\graph}{\Gamma} &=  \NSetI \setminus \{ v \}      \\[.5ex]
          \pmtrue{\closed(Q)}{\graph}{\Gamma} &= \big\{ v \in V \mid 
     \semPath{r}_\graph^v = \emptyset\;\text{for all}\; r\in \NSetR \setminus Q\big\}    &
          \pmtrue{\neg \closed(Q)}{\graph}{\Gamma} &= \big\{ v \in V \mid 
     \semPath{r}_\graph^v \not = \emptyset\;\text{for some}\; r\in \NSetR \setminus Q\big\}              \\[.5ex]
     \pmtrue{\eq(\pi, r)}{\graph}{\Gamma} &= 
\big\{v \in V \mid \semPath{\pi}_\graph^v = \semPath{r}_\graph^v \big\} &
     \pmtrue{\neg \eq(\pi, r)}{\graph}{\Gamma} &=  \big\{v \in V \mid \semPath{\pi}_\graph^v \not = \semPath{r}_\graph^v \big\}     \\[.5ex]
     \pmtrue{\disj(\pi, r)}{\graph}{\Gamma} &=  \big\{ v \in V \mid \semPath{\pi}_\graph^v \cap \semPath{r}_\graph^v  = \emptyset \big\} &       
        \pmtrue{\neg \disj(\pi, r)}{\graph}{\Gamma} &=  \big\{ v \in V \mid \semPath{\pi}_\graph^v \cap \semPath{r}_\graph^v \not = \emptyset \big\}    
  \end{align*}
  \vspace*{-16pt}
  \begin{align*}
     \pmtrue{\varphi_1 \land \varphi_2}{\graph}{\Gamma} &= \pmtrue{\varphi_1}{\graph}{\Gamma}  \cap  \pmtrue{\varphi_2}{\graph}{\Gamma}  \\[.5ex]
         \pmtrue{\geqn{n}{\pi}{\varphi}}{\graph}{\Gamma} &=\big\{ v\in V \mid \#\{v' \in \semPath{\pi}_{\graph}^v \mid v'\in \pmtrue{\varphi}{\graph}{\Gamma}   \} \geq n    \}  \\[.5ex]
    \pmtrue{ \leqn{n}{\pi}{\varphi}}{\graph}{\Gamma} &=\big\{ v\in V \mid \#\{v' \in \semPath{\pi}_{\graph}^v \mid v'\in \pmtrue{\varphi}{\graph}{\Gamma}   \} \leq n    \}  \\[.5ex]
  \end{align*}
  
  \caption{Upper and lower bounds evaluation for shapes. Here $[\cdot]\in \{\ptrue{\cdot}{}{},\mtrue{\cdot}{}{}\big\}$.
}
  \label{fig:expr-eval}
\end{figure*}


We state here the semantics given in \cite{DBLP:conf/kr/OkulmusS24}, adapted for our setting. We assume in this section that shapes are in a normal form where negation is pushed down as far as possible.

A \emph{shape atom} is an expression of the form $S(v)$, where
$S\in \NSetS$ and $v\in \NSetI$.  A \emph{negated shape atom} is
an expression  $\neg S(v)$, where $S(v)$ is an atom. A
\emph{(shape) literal} is a possibly negated shape atom. A
\emph{(3-valued) interpretation} (for SHACL shapes) is any set
$\Gamma$ of shape literals such that there is no~$S(v)\in \Gamma$ with
$\neg S(v)\in \Gamma$. 
Intuitively, an atom $S(v) \in \Gamma$ means that there is
a justification that the shape name $S$ holds at~$v$, $ \neg S(v) \in \Gamma$
means that there is no reason for shape $S$ to hold at~$v$, while $\{S(v),\neg S(v) \} \cap \Gamma = \emptyset$ corresponds to the case where the
satisfaction of $S$ at $v$ is \emph{undefined}.

For a set $K \subseteq \NSetS$ of shape
atoms, let $\neg.K = \{\neg s(a)\mid s(a)\in K\}$.

For a given graph $\graph$ and an interpretation $\Gamma$, we define two
functions $\ptrue{\cdot}{\graph}{\Gamma}$ and $\mtrue{\cdot}{\graph}{\Gamma}$ that map
every shape $\varphi$ to a set of nodes in $\graph$. 
The functions are presented in
Figure~\ref{fig:expr-eval}.

Assume a data graph $\graph$ and a shape $\decl{S}{\varphi}$.
Suppose we ``believe'' a set $\Gamma$ of shape literals, i.e., we assume that all
literals in $\Gamma$ are true. Then $\ptrue{\varphi}{\graph}{\Gamma}$ and
$\mtrue{\varphi}{\graph}{\Gamma}$ return the nodes of $\graph$ where $\varphi$
\emph{is certainly true} and where $\varphi$ \emph{is possibly true},
respectively. Thus $\ptrue{\varphi}{\graph}{\Gamma}$ can be used to infer
positive shape literals: if $v \in \ptrue{\varphi}{\graph}{\Gamma}$, then we
can infer $S(v)$. We can use $\mtrue{\varphi}{\graph}{\Gamma}$ to infer
negative information: if $v\not \in \mtrue{\varphi}{\graph}{\Gamma}$, then we
can infer $\neg S(v)$. These inferences are formalised~next.

\begin{definition}
  Assume a data graph $\graph$ and a shape catalogue $\cat$. We define an operator  $\conseq{\graph,\cat}{\cdot}$ that maps interpretations to
   interpretations as follows:
   \[\conseq{\graph,\cat}{\Gamma} {=} \{S(v) \mid  \decl{S}{\varphi} \in \cat \text{ and } v \in \ptrue{\varphi}{\graph}{\Gamma}\}\]
\end{definition}
We are now ready to define the notion of an unfounded set of shape
atoms.
\begin{definition}[Unfounded set] Assume an interpretation~$\Gamma$, a data
  graph $\graph$, and a shape catalogue $\mathcal{C}$ . A set
  $U$ of shape atoms is called~an~\emph{unfounded set} w.r.t.\,$\Gamma$,
  $\graph$ and $\mathcal{C}$, if
  $v\not \in \mtrue{\varphi}{\graph}{\Gamma\cup \neg.U}$ for
  all $S(v)\in U$, where $\decl{S}{\varphi} \in \cat$.
\end{definition}
\noindent
Assume a graph $\graph$, a set $U$ of shape atoms, and assume the
shape literals in a set $\Gamma$ are true. Intuitively, the
atoms in $U$ form an unfounded set (and can thus be simultaneously set
to \emph{false}) if none of the shape atoms $S(v)\in U$ can possibly
be implied by the associated shape, assuming the negation of the
atoms in $U$ holds (in addition to $\Gamma$ being true).

The following property follows from the fact that $U_1\cup U_2$ is an
unfounded set w.r.t.\,$S$, $\mathcal{G}$ and $\mathcal{C}$ whenever
$U_1, U_2$ are two unfounded sets w.r.t.\,$S$, $\mathcal{G}$ and
$\mathcal{C}$.
\begin{proposition}
  Assume an interpretation~$\Gamma$, a data graph $\mathcal{G}$, and a shape catalogue
  $\mathcal{C}$. There exists a unique set $U$ such
  that:
  \begin{itemize}
  \item $U$ is an unfounded set w.r.t.\,$S$, $\mathcal{G}$ and
    $\mathcal{C}$, and
  \item there is no $U'\supset U$ that is an unfounded set w.r.t.\,$\Gamma$,
    $\mathcal{G}$ and $\mathcal{C}.$
  \end{itemize}
\end{proposition}
The unique set $U$ in the proposition above is called the
\emph{greatest unfounded set} w.r.t.\,$\Gamma$, $\mathcal{G}$ and
$\mathcal{C}$. Assume a data graph $\mathcal{G}$ and a shape catalogue
$\mathcal{C}$. We let $U_{\graph,\cat}$ be the operator that
maps every interpretation $\Gamma$ to the greatest unfounded set
w.r.t.\,$\Gamma$, $\mathcal{G}$ and $\mathcal{C}$. Thus, $U_{\graph,\cat}(\Gamma)$ is
the maximal set of shape atoms that we can safely set to \emph{false}
if we assume that the literals in $\Gamma$ are true. We can now finally
define a well-founded semantics for SHACL.  We define an operator
$W_{\graph,\cat}$ that maps interpretations into interpretations as follows:
It combines the positive consequences based on the $T_{\graph,\cat}$ operator 
and the negated atoms of the greatest unfounded set produced by the
$U_{\graph,\cat}$ operator. More formally, we define:
\[W_{\graph,\cat}(\Gamma)=\conseq{\graph,\cat}{\Gamma} \cup \neg.U_{\graph,\cat}(\Gamma) \] The above
operator is monotone, i.e., $W_{\graph,\cat}(\Gamma_1)\subseteq W_{\graph,\cat}(\Gamma_2)$
whenever $\Gamma_1,\Gamma_2$ are two interpretations with $\Gamma_1\subseteq
\Gamma_2$. Thus $W_{\graph,\cat}$ has the least fixpoint, i.e., there exists $\Gamma$
such that (i) $W_{\graph,\cat}(\Gamma)=\Gamma$, and (ii) there is no $\Gamma'\subset \Gamma$
with $W_{\graph,\cat}(\Gamma')=\Gamma'$. We use $\mathit{WFS}(\graph,\cat)$ to denote the
least fixpoint of $W_{\graph,\cat}$, and call it the \emph{well-founded
  model} of $\graph$ and $\cat$. $\mathit{WFS}(\graph,\cat)$ can be obtained by
constructing a sequence $\Gamma_0,\Gamma_1,\Gamma_2,\ldots$ such that $\Gamma_0=\emptyset$
and $\Gamma_{i+1}=W_{\graph,\cat}(\Gamma_i)$ until eventually an index $j$ is reached
were $\Gamma_j=\Gamma_{j+1}$. Then $\mathit{WFS}(\graph,\cat)=\Gamma_j$.


We are now ready to give the definition of the well-founded semantics as introduced in \Cref{def:semanticsSHACL}, leading to the unique shape assignment \wfass.

\begin{definition}
    Given a data graph $\graph$ and shape catalogue $\cat$, the well-founded shape assignment \wfass is defined as follows: 
    \[
\wfass(S) = \{v \in V \mid S(v) \in \mathit{WFS}(\graph,\cat)  \} \text{ where } S \in \Sigma(\cat) \cap \NSetS \ .
\] The well-founded semantics \wfsem is then simply $\wfsem = \{ \wfass \}$.
\end{definition}

\paragraph{Supported Semantics.}

The supported semantics (or supported \emph{model} semantics) were defined in \cite{CRS18}, where they were also defined with a focus on SHACL validation in the presence of target expressions. For our setting, where we do not have targeting, it can be given by a relatively straightforward definition. 

\begin{definition}
    Given a data graph $\graph$ and shape catalogue $\cat$, the supported semantics \supsem  is defined as follows: 
    \[
\supsem = \{\alpha \mid  \text{ for every } \decl{S}{\varphi} \in \cat \text{ we have } \semdelta{\varphi}{\alpha}{\graph} = \{ v \mid v \in \alpha(S) \}     \}.
\]
\end{definition}

\paragraph{Stable Semantics.}

The stable semantics (or stable \emph{model} semantics) were defined in \cite{ACORSS20}. It essentially extends the supported semantics by one additional property, requiring a so-called ``level assignment'' of the shape assignment. 
We assume here that we also have shapes of the form $\varphi \lor \varphi$, where this is just syntactic sugar for $\neg (\neg \varphi \land \neg \varphi )$.

\begin{definition}
    Given a data graph $\graph$ and a shape assignment $\alpha$, a \emph{level assignment} for $\graph$  and $\alpha$ is a function $\mathrm{level}$ that maps tuples in $\{  (\varphi,v) \mid v \in \semdelta{\varphi}{\alpha}{\graph} \}$ to integers, and satisfies the following conditions: 
\begin{itemize}
    \item $\mathrm{level}(\varphi_1 \land \varphi_2,v) = \max{\{ \mathrm{level}(\varphi_1,v), \mathrm{level}(\varphi_2,v) \}}$,
    \item $\mathrm{level}(\varphi_1 \lor \varphi_2,v) = \min{\{ \mathrm{level}(\varphi_1,v), \mathrm{level}(\varphi_2,v) \}}$,
    \item $\mathrm{level}(\geqn{n}{\pi}{\varphi},v) $ is the smallest $k \geq 0$ for which there are $n$ nodes $b_1, \dots , b_n$ such that $\mathrm{level}(\varphi,b_i) \leq k$, $(v,b_i) \in \semPath{\pi}_\graph$ and $b_i \in \semdelta{\varphi}{\alpha}{\graph}$ for all $1 \leq i \leq n$ and
    \item $\mathrm{level}(\leqn{n}{\pi}{\varphi},v) $ is the largest $k \geq 0$ for which there are $m$ nodes with  $m \leq n$, $b_1, \dots , b_m$ such that $\mathrm{level}(\varphi,b_i) \leq k$, $(v,b_i) \in \semPath{\pi}_\graph$ and $b_i \in \semdelta{\varphi}{\alpha}{\graph}$ for all $1 \leq i \leq m$.    
\end{itemize}    
\end{definition}

With this notion, the definition of the stable semantics is as follows.

\begin{definition}
    Given a data graph $\graph$ and shape catalogue $\cat$, the stable semantics \stsem  is defined as follows:

\begin{align*}
    \stsem = \left \{ \alpha \in \supsem \middle| 
    \begin{aligned}
          &\text{ there is a level assignment } \mathrm{level}  \text{ s.t.} \\ &\text{ for all } v \in \alpha(S)  \text{ with } \decl{S}{\varphi} \in \cat \\ 
          & \text{ we have } \mathrm{level}(\varphi,v) <  \mathrm{level}(S,v) 
    \end{aligned} \right \}
\end{align*}

\end{definition}


\end{toappendix}


\begin{definition}\label{def:semanticsSHACL}
    Given an adorned graph $\graph\!$ and shape catalogue $\cat\!$, we denote  the sets of assignments for $\graph\!$ and $\cat$ by $\supsem$ under SUS, by $\stsem$ under STS and by $\wfsem$ under WFS. The latter
%
    consists of the unique assignment $\wfass$.
\end{definition}
%

\begin{example}
\label{ex:SHACLsemantics}
	We illustrate the three semantics on the graph $\graph$  and 
	catalogue $\cat$:
%
%
%
%
\\[3pt]  	\hfill
	\parbox[c]{0.58\textwidth}
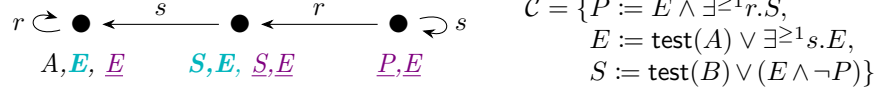
{
		\centering
	\begin{tikzpicture}[scale=0.7,->, >=Stealth, node distance=3cm, every node/.style = {font=\normalsize}]
		\node [
		 label={[label distance=0mm]270:A,{\color{Aquamarine}\textbf{E}}, {\color{violet}\underline{E}}  }] (alex) at (0,0) {  \nodebullet};

		\node [label={[label distance=0mm]270:{\color{Aquamarine}\textbf{S,E}, \color{violet}\underline{S},\underline{E}}  }] (bob) at (3,0) { \nodebullet};        
		\node [label={[label distance=0mm]270:{\color{violet}\underline{P},\underline{E}}  }] (drew) at (6,0) { \nodebullet};
		
		\path (alex) edge [loop left] node[left] {$r$} (alex);
		
		\path (bob) edge  node[above] {$s$} (alex);
		\path (drew) edge  node[above] {$r$} (bob);
		\path (drew) edge [loop right] node[right] {$s$} (drew);
	\end{tikzpicture}
	}
	\hfill
	\parbox[c]{0.38\textwidth}{
	$\cat=\{ \decl{P}{E \land \exists^{\geq 1} r. S},$ 
	\\ $\phantom{\cat=\{ }\decl{E}{\test(A) \lor \exists^{\geq 1} s. E},$ 
	\\ $\phantom{\cat=\{ }\decl{S}{\test(B) \lor ( E \land \neg P})\}$
		}~~
	\\[3pt]
The parts defined in the graph, common to all semantics, are printed in black. The unique well-founded assignment $\wfass$ is in {\color{Aquamarine} \textbf{aquamarine}}, while  {\color{violet}\underline{violet}} shows an assignment $\alpha$ that is both in $\supsem$ and $\stsem$. \co{To highlight the difference between the two, note that under WFS, the right node does not get assigned E, since WFS does not permit such ``cyclic supports'', whereby the node is assigned E because it has an s-edge to a node assigned E, which happens to be itself.}
\end{example}

We rely on the following known complexity results for SHACL.
%
\begin{theorem}[\cite{ACORSS20,CRS18,DBLP:conf/kr/OkulmusS24}]\label{thm:LFPTime}
    For every data graph $\graph$ and every shape catalogue $\cat$, it holds that:
\begin{enumerate*}[label=\arabic*., itemjoin=\quad]
        \item The well-founded assignment $\wfass$ is unique and can be computed in time polynomial in $\graph$ and $\cat$. 
        \item Given an assignment $\alpha$, testing whether $\alpha\in \supsem$ or $\alpha\in \stsem$ takes 
        time polynomial in $\graph$ and $\cat$.
    \end{enumerate*} 
\end{theorem}

\section{Introducing Fitting Problems for SHACL} 
\label{sec:FittingProblems}

We introduce a variant of the fitting problem for  graphs, where some nodes are marked as positive or as negative examples, adapted to the SHACL setting. The goal is to compute a shape expression $\varphi$, 
possibly using shape names defined in a given catalogue and under under a specific semantics,  
such that the evaluation of $\varphi$ contains the positive examples while it excludes the negative ones.   
\begin{definition}\label{def:fittingProblem}%
     A \emph{fitting instance} is a tuple 
    $\mathcal{F} = (\graph, \cat, \PEx, \NEx, \Sigma)$, 
    where $\graph = (V,E,\ell)$ is a data graph, $\cat$ is a shape catalogue,  
    $\PEx$ is the set of \emph{positive} nodes and $\NEx$ is the set of \emph{negative} nodes, s.t.\   $\PEx \cup \NEx \subseteq V$, and $\Sigma$ is a signature.
    Given some target language $\Lmc$ and some semantics $\Sem$, 
    an
    \emph{$\Lmc$ fitting under\ $\Sem$} for $\mathcal{F}$ is an $\Lmc$ shape $\varphi$ with $\Sigma(\varphi) \subseteq \Sigma$ s.t.
     there exists some assignment $\alpha
      \in \SemGra$ with $p \in \semdelta{\varphi}{\alpha}{\graph}$
          for every $p \in \PEx$, and $n \not \in \semdelta{\varphi}{\alpha}{\graph}$ for every $n \in \NEx$.
\end{definition}


For a fitting 
instance $\mathcal{F} = (\graph, \emptyset, \PEx, \NEx, \Sigma)$
with an empty catalogue $\cat$, 
we may write $\mathcal{F} = (\graph, \PEx, \NEx, \Sigma)$. 
Note that in this case, \bg{our instance contains no defined shapes, so }all semantics give only the empty assignments. Thus the choice of semantics becomes irrelevant. 
If $\Sigma = \Sigma(\graph) \cup \Sigma(\cat)$, we may write  $\mathcal{F} = (\graph, \PEx, \NEx)$.


A fitting instance may have very many fittings. In this case, the most informative fittings are often preferred. We formalise these next.
\begin{definition}
    \label{def:MSFittingProblem}
Let $\varphi$ and $\varphi'$ be two $\Lmc$ shapes  and \Sem a semantics. Shape  \emph{$\varphi$ is more specific than $\varphi'$} (written $\varphi \subseteq \varphi'$), if for every data graph $\graph$, shape catalogue $\cat$ and every shape assignment $\alpha\in\SemGra$, we have  $\semdelta{\varphi}{\alpha}{\graph} \subseteq \semdelta{\varphi'}{\alpha}{\graph}$.
%
    Given a fitting instance $\Fmc$ and some $\Sem$, a \emph{most specific \Lmc fitting (\Lmc MSF) $\varphi$ under \Sem } is 
    an \Lmc fitting for $\mathcal{F}$ under \Sem  and for every \Lmc fitting $\varphi'$ for $\mathcal{F}$ under \Sem, we have $\varphi \subseteq \varphi'$.    
\end{definition}
\bg{Given a fitting instance $\mathcal{F}$ and target language $\Lmc$, the \emph{$\Lmc$ fitting (resp. the $\Lmc$ MSF) existence problem} is a decision problem that takes as input an instance $\mathcal{F}$, and asks whether $\mathcal{F}$ permits an $\Lmc$ fitting (resp. MSF). }
\bg{The corresponding \emph{$\Lmc$ MSF computation problem} returns the $\Lmc$ MSF, if it exists.}

For an assignment $\alpha$ and a pure data graph $\graph$, the  \emph{$\alpha$-adornment} of $\graph$ is the adorned data graph $\graph^\alpha = (V,E,\ell')$ with
$ \ell'(v) = \ell(v) \cup \{ S \in \NSetS  \mid v \in \alpha(S) \} \text{ for } v \in V.$
For the special case of $\alpha\in\wfsem$, we use $\graph^\mathsf{wf}$ to denote the unique $\graph^\alpha$.    
%
Clearly, for every shape expression $\varphi$, for every assignment $\alpha$, and every data graph $\graph$, 
the evaluation of $\varphi$  under $\alpha$ coincides with the \bg{evaluation in the} adorned data graph, i.e.
$\semdelta{\varphi}{\alpha}{\graph} = \semdelta{\varphi}{}{\graph^\alpha}$. 
From this, we get:
\begin{lemma}
\label{lemma:FittingEmptyCat}
A $\Lmc$ shape $\varphi$ is a (most specific) $\Lmc$ fitting for $(\Gmc,\Cmc,\PEx,\NEx,\Sigma)$ under $\Sem$ iff $\varphi$ is a (most specific) $\Lmc$ fitting for $(\Gmc^{\alpha},\emptyset,\PEx,\NEx,\Sigma)$ for some  $\alpha \in\SemGra$.
\end{lemma}
So, once a pure graph is transformed into an adorned graph, the shape catalogue $\cat$ can be omitted when solving (most specific) fitting problems.
In these instances, the catalogue is empty and therefore the WFS, STS and SUS coincide. This means the problem can be considered independently of semantics.


\paragraph{Restricting the fitting language.} 
For this initial work, 
we 
\bg{consider} a  restricted fragment of shape expressions borrowed 
from the field of Description Logic (DL) \bg{as our target language}. 
There are several reasons to cover only a fragment of SHACL shapes. The most apparent is that disjunction and nominals in the target language essentially render fittings to enumerations of the positive examples, trivialising the problem.
The DLs \EL and \ELI have been explored as target languages for fitting problems, albeit to attain concepts and variations of conjunctive queries. We extend these two DLs with a basic form of path expressions $r^*$ that can express finite, but arbitrarily long paths over a role $r$. 
\noindent
For readability, we 
shorten  $\hasvalue(X)$ by $X$, $\exists^{\geq 1}r$ by $\exists r$, and $\exists (r)^*$ by $\exists r^*$. 
Hence, we can
write 
\bg{a shape declaration $\decl{S}{\hasvalue{(X)}\land\exists^{\geq 1}(r)^*.A}$ as $\decl{S}{X\land\exists r^*.A}$.}
%
\bente{Is it worth it to change the \ourDL shape notation to be $\varphi$ a sis in line with previously defined SHACL shapes? I do agree with earlier suggesitons that it is maybe a bit odd.}
\begin{definition}
An \emph{\ourDL shape} $C$ is a SHACL shape given by the grammar: 
	\[
         C   \rightarrow  
         \top \gMid  S  \gMid X \gMid \exists r.C \gMid \exists r^*.C \gMid C \land C  \ \text{,}
	    \]
    where $S \in \NSetS $, $X \in \NSetC $, and $r \in \NSetRMinus$.
An \ELstar shape is an \ELIstar shape with $r\in\NSetR$.
An \emph{\ELI shape} is an \ourDL shape with no occurrence of $r^*$ for any $r \in \NSetRMinus$. 
An \emph{\EL shape} is an \ELI shape with $r\in\NSetR$.
\end{definition}

\ourDL shapes can be conveniently represented as \bg{finite node- and edge-labelled} trees $\Tmc=(V,E,\ell)$, analogous to the \EL  case~\cite{BKM-IJCAI-99}. 
In such \bg{\emph{\ourDL description trees}, edges may be labelled with} special symbols $\roleStarSymbol{r}$ to represent existential shapes containing the Kleene star.
Hence, we let $\NSetRStar = \NSetRMinus \cup \{\roleStarSymbol{r} \mid r\in\NSetRMinus\}$. For a signature $\Sig$, we define its \emph{extended role signature} as 
$\exSigStar = \{r, \roleStarSymbol{r}, r^-, \roleStarSymbol{r^-} \mid r \in \SigR \}$
and  
$\exSigStar = \{r, \roleStarSymbol{r}, r^-, \roleStarSymbol{r^-} \mid r \in \SigR (X)\}$
for the extended role signature of $X$.

The \emph{\ourDL description tree} of a shape 
$C = A_0 \land \dots \land A_n \land \exists R_1.C_1 \land \dots \land \exists R_m.C_m$
is inductively defined as 
the finite edge- and node-labelled tree $\Tmc_C = (V_C,E_C,\ell_C)$ 
 where the root node $\varepsilon$ is labelled with the class names $A_0, \ldots, A_n$, and for each 
 $1 \leq j \leq m$, it has a child  $j$ with $(\varepsilon,R_j,j)$ and $j$ the root of (an isomorphic copy of) the description tree of $C_j$;
see \Cref{fig:treeLabelling}.
\bg{We call $C$ the \emph{corresponding shape} of $\ourDLTree_C$.
We sometimes consider the \emph{pointed} graph $(\ourDLTree,u_0)$ of a tree $\ourDLTree$ with its root $u_0$; slightly abusing notation, we denote this pointed graph as $\ourDLTree$.}

%
\nosectionappendix
\begin{toappendix}

\section{Definitions for Section \ref{sec:FittingProblems} (Introducing Fitting Problems for SHACL)}
\begin{definition}[\ourDL Description Tree] 
\label{def:descTree} An 
\emph{\ourDL description tree} is a finite edge- and node labelled tree-shaped graph $\Tmc = (V,E,\ell)$ with $E \subseteq V \times \NSetRStar \times V$  and $\ell : V \rightarrow (2^\NSetC )$. 
Let $C$ be an \ourDL shape w.l.o.g.\ in the form $A_0 \land \dots \land A_n \land \exists R_1.C_1 \land \dots \land \exists R_m.C_m$ with $A_i \in \NSetC$, where $0 \leq i \leq n$ and each $C_j$ is an \ourDL shape and $R_j \in \{ r, r^* \mid r \in \NSetRMinus \}$, with $1 \leq j \leq m$. 
Given an \ourDL shape $C$, we define its \emph{role depth} $\depth(C)$ as the maximal nesting depth of its quantifiers.

$C$ can be translated into its \emph{\ourDL description tree $\Tmc_C= (V_C,E_C,\ell_C)$} by an induction on the role depth of $C$ as follows. Let $u_0$ be the root of $\Tmc_C$. 
If $\depth(C) = 0$, then $V_C = \{u_0\}$, $E_C = \emptyset$ and  $\ell_C(u_0) = \{ A_0, \dots, A_n \}$. \\ 
\noindent
If $\depth(C) > 0$, for $1 \leq j \leq m$, let $\Tmc_{C_j} = (V_j,E_j,\ell_j)$ be the inductively defined 
description tree corresponding to $C_j$ and $u_j$ be its root, where we assume that \bg{$V_1, ... V_n$}
are pairwise disjoint, then 
\vspace{-3mm}
\begin{align*}
    V_C &=  \{ u_0\} \cup \bigcup_{1\leq j \leq m} V_j   \quad \quad     \ell_C(v) = 
    \begin{cases} 
     \{ A_0, \dots, A_n \}  & v = u_0 \\ 
     \ell_j(v)              & v \in V_j, 1 \leq j \leq m
    \end{cases}\\ 
         f(R) &= \begin{cases}
        r_\star & \!\!\text{if } R = r^* \text{and } r \in \NSetRMinus \\  
        R & \!\!\text{otherwise}  \\  
    \end{cases} 
    \qquad 
    E_C =   \!\bigcup_{1 \leq j \leq m}\! \{ u_0, \mathit{f}(R_j),u_j \} 
    \cup  \bigcup_{1 \leq j \leq m}\! E_j
\end{align*}
\end{definition}
\end{toappendix}
\section{Complexity of 
\ourDL
Fitting Existence}
\label{sec:FittingExistence}

We study the complexity of deciding the existence of an \ELIstar-fitting for an instance $\fittinginst$, and show that the problem is \textsc{ExpTime}-complete in size of \graph under the three semantics specified in \Cref{sec:preliminaries}. 
Deciding \EL or \ELI fitting existence is known to \textsc{ExpTime}-hard \cite{cate2025extremalfittingproblemsconjunctive,DBLP:conf/ijcai/FunkJLPW19}, 
but this need not transfer to \ourDL, since more expressive fitting languages generally admit more fittings.
While the existence of an $\ELI$ fitting trivially implies  existence of an \ourDL fitting, the  converse does not hold. 
\begin{example} \label{ex:ELIStarFitting}
Consider $\fittingNoCat$ 
with nodes in $\PEx$ marked  \nodepositiveText and nodes  in	$\NEx$ marked  \nodenegativeText below. While $\fit$ has no \ELI fitting, $\exists \roleStarOperation{r}\!.A$ is an \ourDL fitting.  
\vspace{-\medskipamount}
\begin{center}
        \begin{tikzpicture}[scale=0.7,->, >=Stealth, node distance=2cm, every node/.style = {font=\normalsize}]

        \node [label={[label distance=-2mm]0:A}] (a) at (0,0) {\nodepositive};

        \node (a2) at (2,0) {\nodepositive};
        \node [label={[label distance=-2mm]0:A}] (b2) at (4,0) { \nodebullet};

        \node [label={[label distance=-2mm]0:B}] (a) at (6,0) {\nodenegative};

        
    
        \path (a2) edge node[above] {$r$} (b2);
        \end{tikzpicture}
\end{center}
\vspace{-\bigskipamount}
\end{example}
Despite the increased expressivity, the \ourDL-fitting existence problem is not easier in the worst case. The hardness holds even without 
signature restrictions and for an empty shape catalogue, and hence independently of the choice of semantics. This follows from Theorem 12 in \cite{DBLP:conf/ijcai/FunkJLPW19}, which extends to our setting. 
%
\begin{propositionrep}
\label{prop:FittingExpHard}
    Deciding the 
    existence of an \ourDL-fitting 
    for a given \fittingNoCatSig
    is \textsc{ExpTime}-hard.
\end{propositionrep} 
%

\begin{toappendix}

\begin{proofsketch} 
To show that fitting existence is \textsc{ExpTime}-hard, we reduce from the word problem for an alternating, linear space bounded Turing machine $M$. 
We reuse the reduction by 
\cite{DBLP:conf/ijcai/FunkJLPW19} 
(Theorem 12) for the $\EL$-fitting existence problem. We 
build from $M$ and an input word $w$, an instance $\fit_M = (\graph_M, V_w^+, \{n\})$. $\graph_M$ is the union of two graphs. In a nutshell, the nodes of the first graph represent possible configurations of $M$ restricted to a single tape position, and their connections reflect the local effects of executing the transitions of $M$. The second graph represents how acceptance propagates along existential and universal states in a run of $M$. The positive examples represent the initial tape contents in such a way that possible fittings represent correct executions of $M$ on $w$, while the single negative example restricts the fittings to non-accepting runs. The authors show that there exists an \EL-fitting for $\fit_M$ iff there is no accepting run of $M$ on $w$. To lift this claim from $\EL$ to \ourDL, we leverage a property of $\graph_M$: by construction, all paths in $\graph_M$ that start at some node in $V_w^+$ move along edges whose labels strictly alternate between (roles representing) existential and universal states of $M$, until an accepting or rejecting state is reached. \bg{Because of this strict alternation, $r^*$-edges can only represent $r$-paths of length exactly one, and hence} it is possible to show that for every \ourDL-fitting for $\fit_M$, one can define a corresponding $\EL$-fitting for $\fit_M$ \bg{by replacing any $r^*$-edges by regular $r$-edges.} Once this has been established, we can argue as in \cite{DBLP:conf/ijcai/FunkJLPW19} that there exists an \ourDL-fitting $C$ for $\Fmc_M$ iff $M$ does not accept $w$. \qed
\end{proofsketch}
\end{toappendix}


We obtain matching upper bounds under the semantics mentioned in \Cref{sec:preliminaries}.
We first study fitting instances without catalogue. Note that in this setting the choice of semantics is irrelevant, so we omit it. Then we obtain results for instances in presence of catalogues by reducing it to the setting with adorned graphs, and show that these results hold under WFS, SUS and STS.
To characterize fittings we use 
pre-matches, which are homomorphisms from \ourDL trees to pointed graphs \bg{taking into account reflexive transitive closure}.


\begin{definition} 
\label{def:prematch}
A \emph{pre-match} for an \bg{\ourDL description} tree $\ourDLTree = ( V, E,\ell)$ with root $v_0$ into a pointed graph $(\graph,u)$, 
is a function $\rho: V  \rightarrow \graph$ such that $\rho(v_0) = u$ and for all $v,v'\in V$: 
\begin{enumerate*}[label=\arabic*., itemjoin=\quad]
\item $\rho(v) \in \semPath{A}_\graph $ for every $ A \in \ell(v)$,
\item if $r(v,v') \in E$ then $(\rho(v),\rho(v')) \in \semPath{r}_\graph $, and
\item if $\roleStarSymbol{r}(v,v') \in E$ then $(\rho(v),\rho(v')) \in \semPath{r^*}_\graph $. 
\end{enumerate*}
\end{definition}
\begin{lemma} \label{lm:preMatchCapturesEval}
For every $\ourDL$ shape $C$, there is a pre-match $\rho$ from its \ourDL description tree  $\ourDLTree_C$ with root $v_0$ into $(\Gmc, u)$ iff $ \Gmc \models C(\ayt{\rho(v_0)})$ s.t. $\rho(v_0)=u$.
\end{lemma}


Our algorithm for deciding fitting existence uses tree automata to 
\bg{detect pre-matches from \ourDL description trees of potential fitting shapes to the fitting instance.}
Since the automata run on 
\bg{infinite,} (strictly) node-labelled full $k$-ary trees, we rewrite our \ourDL description trees into this form.

Let $\labelLang = (\{ - \}  \cup  \NSetRStar) \times 2^{\SigC}$ be an alphabet.
For $w \in \labelLang$ with $w = (r,\mathbf{A})$, we refer to its components by  $w.R = r$ and $w.\mathbf{A}= \mathbf{A}$.
For $k\geq 0$, a \emph{node-labelled \ourDL $k$-tree} is a $k$-ary tree $\nodeLabelTreeWithoutC=(V,\ell)$ whose labelling function is such that $\ell : V \rightarrow \labelLang$. Such a tree $\nodeLabelTreeWithoutC$ is \emph{proper} if its root $\varepsilon$ has $\ell(\varepsilon) = (-,\mathbf{A})$ for some $\mathbf{A}\neq\{\top\}$, and for each node $w \neq \varepsilon$ with $w.R = -$, the labelling is s.t.
$\ell(w) = (-,\{\top\})$ \bg{and every branch in $\nodeLabelTreeWithoutC$  reaches a node $(-,\{\top\})$ in finite steps}. 

Given an \ourDL description tree $\ourDLTree$ and $k\geq deg(\ourDLTree)$, we transform this into a proper node-labelled \ourDL $k$-tree $\nodeLabelTreeWithoutC$ by moving the edge-labels of $\ourDLTree$ into the successor node-labels of $\nodeLabelTreeWithoutC$, and otherwise keeping the node-labels the same. Additionally, for every node in $\ourDLTree$ with less than $k$ children, we add ``dummy'' child nodes labelled $(-,\{\top\})$  in $\nodeLabelTreeWithoutC$ so that it becomes a $k$-ary infinite tree.

If 
$\nodeLabelTree$ is the corresponding node-labelled tree for some $\ourDLTree_C$, we refer to $C$ as its corresponding shape.
As each proper node-labelled \ourDL $k$-tree $\nodeLabelTreeWithoutC$ is the encoding 
of some $\ourDL$ description tree $\ourDLTree$, we can say that $\nodeLabelTreeWithoutC$ has a pre-match into a pointed graph whenever 
$\ourDLTree$ does.  \Cref{fig:treeLabelling} illustrates the two tree notions. 
\vspace{-\medskipamount}

\begin{toappendix}
\begin{definition}[Node-labelled \ourDL $k$-Trees]
\label{def:elitree}
Let $\labelLang = (\{ - \}  \cup  \NSetRStar) \times 2^{\SigC}$ be an alphabet.
For 
$w \in \labelLang$ with $w = (r,\mathbf{A})$, we refer to its components by  $w.R = r$ and $w.\mathbf{A}= \mathbf{A}$.

Let $k \geq 0$. A \emph{node-labelled \ourDL $k$-tree} 
$\nodeLabelTreeWithoutC= (V,\ell)$ consists of node set $V$ containing all words over the alphabet $\{1,\ldots,k\}$
and  of a  labelling function $\ell : V \rightarrow \labelLang$. 
We call such a tree \emph{proper} if the root $\varepsilon$ has $\ell(\varepsilon) = (-,\mathbf{A})$ for some $\bg{\mathbf{A}\neq\{\top\}}$, and for each for every $w \in V$, where $w \neq \varepsilon$ and $w.R = -$, the labelling is such that 
$\ell(w) = (-,\{\top\})$, \bg{and every branch in $\nodeLabelTreeWithoutC$  reaches a node labelled $(-,\{\top\})$ in a finite number of steps}.
If  the labels in $\nodeLabelTreeWithoutC$ contain only symbols from signature $\Sigma$,  $\nodeLabelTreeWithoutC$ is called a tree \emph{over $\Sigma$}.

Given an \ourDL description tree $\ourDLTree = (V_{\ourDLTree}, E_{\ourDLTree}, \ell_{\ourDLTree})$ and $k \geq deg(\ourDLTree)$. Let $f$ be an arbitrary but fixed function that assigns to each  $v \in V_{\ourDLTree}$ a word over $\{1,\ldots,k\}$ such that: 
\begin{enumerate*}[label=\arabic*., itemjoin=~]
\item for  root $v_0$ of $\ourDLTree$: $f(v_0) = \varepsilon$  and 
\item  for the children $v_1,\ldots,v_n$ of $v$ and $1 \leq i \leq n$: $f(v_i) = f(v)i$. 
\end{enumerate*} 
The encoding
$\nodeLabelTreeF{k}{\ourDLTree}$ of $\ourDLTree$ is   the node-labelled $k$-ary tree $\nodeLabelTreeWithoutC=(V_{\nodeLabelTreeWithoutC},\ell_{\nodeLabelTreeWithoutC})$ that labels:
\vspace{-\medskipamount}
\begin{align*}
    \ell_{\nodeLabelTreeWithoutC}(\varepsilon) =~& (-,\ell_{\ourDLTree}(v_0)) & \text{for the root $v_0$ of $\ourDLTree$} 
    \\ 
        \ell_{\nodeLabelTreeWithoutC}(wj) =~& (r,\ell_{\ourDLTree}(v_j)) & \text{for each $wj = f(v_j)$ with $(v,r,v_j) \in E_{\ourDLTree}$} 
        \\ 
            \ell_{\nodeLabelTreeWithoutC}(w) =~& (-,\{\top\}) & \text{if there is no $v \in V_{\ourDLTree}$ with $f(v) = w$}
     \vspace{-\bigskipamount}
\end{align*} 

If a proper node-labelled \ourDL $k-$tree is the encoding $\nodeLabelTreeF{k}{\ourDLTree_C}$, we write $\nodeLabelTree$, and we refer to $C$ as the corresponding shape of $\nodeLabelTree$.
 \end{definition}
\vspace{-\smallskipamount}
\end{toappendix}

\begin{figure}[t]
    \centering
\resizebox{!}{2.5cm}{    
\begin{tikzpicture}[scale=0.6,-, >=Stealth, sibling distance=35mm, every node/.style = {font=\normalsize}]
\node  (v1) at (-1.5,0) {$\Tmc_C\textnormal{:}$};
  \node[outer sep=0mm,inner sep=-2mm,label=right:{\phantom{a}$A$}] { \nodebullet}
    child {
    node[outer sep=0mm,inner sep=-2mm,label=right:{}] (left node) { \nodebullet} 
     child {
                node[outer sep=0mm,inner sep=0mm,label= right:{$A,B$}] (right node) { \nodebullet}
                child[white]  {
                    node{}
                }
                edge from parent node[right] {$t$}
            }
     edge from parent node[right] {\phantom{.}$\roleStarSymbol{r}$}
    }
    child {
        node[outer sep=0mm,inner sep=-2mm,label=right:{\phantom{a}$B'$}] (right node) { \nodebullet}
        edge from parent node[right] {\phantom{a}$s$}
        };
\end{tikzpicture} 
}
\resizebox{!}{2.5cm}{
\begin{tikzpicture}[scale=0.68,-, >=Stealth, 
level 1/.style={sibling distance=8cm},
level 2/.style={sibling distance=4cm}, 
level 3/.style={sibling distance=2cm,level distance=1cm}, 
level distance=2cm,
every node/.style = {font=\normalsize}]
\node  (v1) at (-1.5,0) {$\nodeLabelTree\textnormal{:}$};
  \node[circle, draw,minimum size=8mm,label=right:{\phantom{a}$(-,\{A\})$}] { $\varepsilon$}
    child {
      node[circle, draw,minimum size=8mm,label=right:{\phantom{a}$(\roleStarSymbol{r},\emptyset)$}] (left node) { $1$} 
      child {
        node[circle, draw,minimum size=8mm,label= right:{$(t,\{A,B\})$}] { $11$} 
        child[dotted] {node {}}
        child[dotted] {node {}}
      }
      child {
        node [circle, draw,minimum size=8mm,label=right:{$(-,\{\top\})$}]  {$12$}
        child[dotted] {node {}}
        child[dotted] {node {}}
      }
    }
    child {
      node[circle, draw, minimum size=8mm, label=right:{$(s,\{B'\})$}] { $2$}
      child {
        node [circle, draw,label=right:{$(-,\{\top\})$}]  {$21$}
        child[dotted] {node {}}
        child[dotted] {node {}}
      }
      child {
        node [circle, draw, minimum size=8mm, label=right:{$(-,\{\top\})$}]  {$22$}
        child[dotted] {node {}}
        child[dotted] {node {}}
      }
    };
\end{tikzpicture}
}
\caption{The  \ourDL description tree $\ourDLTree_C$ and  the node-labelled $2$-tree $\nodeLabelTree$ of the \ourDL shape 
   \mbox{$C = A \land \exists \roleStarOperation{r}\!.\exists t.(A \land B) \land \exists s.B' $}.}
    \label{fig:treeLabelling} 
    \vspace{-2mm}
\end{figure}

\begin{toappendix}
\paragraph{Two-way alternating parity tree automata.}

We recall the basics of two-way alternating automata on infinite node-labelled $k$ trees. 
For the presentation, we largely follow~\cite{DBLP:conf/dlog/CalvaneseGL02} and~\cite{cate2025extremalfittingproblemsconjunctive}.

\begin{definition}[$k$-ary Two-Way Alternating Parity Tree Automata.]
Let $\mathcal{B}(I)$ be the set of positive Boolean formulae over a set of variables $I$, built inductively by applying $\land$ and $\lor$ starting from $\bot$, $\top$ and the elements of $I$. For a set $J \subseteq I$ and a formula $C \in \mathcal{B}(I)$, we say $J$ \emph{satisfies} $C$ iff assigning $\top$ to the elements in $J$ and $\bot$ to those in $I \subseteq J$, makes $C$ true. For a positive integer $k$, let $[k] = \{-1, 0,1, \dots,k \}$. 
A \emph{$k$-ary two-way alternating parity tree automaton} ($k$-2ATA) running over \ourDL trees with degree $k$, 
is a tuple $\automaton = \langle \Sigma, Q, \delta, q_0, p \rangle$, where $\Sigma$ is an input alphabet, $Q$ is a finite set of states, $\delta: Q \times \Sigma \rightarrow \mathcal{B}([k]\times Q)$ is the transition function, $q_0 \in Q$ is the initial state and $p: Q \rightarrow \mathbb{N}$ is a priority function. 
\end{definition}

\begin{definition}[$k$-2ATA Run]
A \emph{run} of a $k$-2ATA $\mathfrak{A}$ over an
 \ourDL tree $\nodeLabelTree=(V,\ell)$ with degree $k$ is a $(V \times Q)$-labelled tree $\mathfrak{T}_{\run}=(V_{\run}, \ell_{\run})$ satisfying:
\begin{enumerate}
    \item $\varepsilon \in V_{\run}$ and $\ell_{\run}(\varepsilon) = (\varepsilon,q_0)$. 
    \item Let $y \in V_{\run}$ with $\ell_{\run}(y) = (x,q)$ and $\delta(q,\ell(x)) = C$. Then there is a (possibly empty) set $S = \{ (c_1,q_1), \dots, (c_n,q_n) \} \subseteq [k] \times Q$ such that: 
    \begin{itemize}
        \item $S$ satisfies $C$, and 
        \item for all $1 \leq i \leq n$, we have that $y \cdot i \in V_{\run}$, $x \cdot c_i \in V$ and $\ell_{\run}(y  \cdot i) = (x \cdot c_i, q_i)$.
    \end{itemize}
\end{enumerate}
Given an infinite path $p \subseteq V_{\run}$, let $\mathit{inf}(p) \subseteq Q$ be the set of states that appear infinitely often in $p$. A run $\mathfrak{T}_{\run}$ of a $k$-2ATA $\mathfrak{A} = \langle \Sigma, Q, \delta, q_0, p \rangle$ is \emph{accepting} if for every infinite path $p \subseteq V_{\run}$, the state with the highest priority 
in $\mathit{inf}(p)$ is even. \bg{Note that a finite run is always accepting.}

\bg{A \emph{partial run} of a $k$-ATA $\mathfrak{A}$ is a run that does not satisfy Condition 1 above. If a partial run has as its root a state $(w,q)$, we call it a \emph{partial run rooted at $(w,q)$}, or, for short, a \emph{partial $(w,q)$-run}.}
\end{definition}

\end{toappendix}






\paragraph{Deciding fitting existence with 2ATAs.}
Next, we want to construct a \emph{two-way alternating parity tree automaton} (2ATA) $\automaton_\Fmc$ that accepts exactly the \ayt{proper}
node-labelled $k$-ary trees that represent a 
fitting for some $\fit$. We let a 2ATA $\automaton$ be defined as usual, as a tuple $\automaton = \langle \Sigma, Q, \delta, q_0, p \rangle$, where $\Sigma$ is the input alphabet, $Q$ a finite set of states with an initial state $q_0$, a priority function $p$ mapping states to natural numbers, and a transition function $\delta$ mapping pairs of states and words in $\Sigma$ to some positive Boolean formulae over pairs of states and integers up to some $k$. We call a 2ATA that accepts only $k$-ary trees a \emph{$k$-2ATA}.

For a pointed graph $(\graph,v_0)$ and $k \geq 0$, we first define 
a $k$-2ATA $\automaton^{\graph, v_0}$,
\bg{which takes as input proper node-labelled \ourDL $k$-trees and accepts exactly those}
with a pre-match into $(\graph,v_0)$. 
Let $\graph = (V,E,\ell)$ and $v_0 \in V$. 
The $k$-2ATA $\mathfrak{A}^{\graph, v_0} = \langle \labelLang, Q, \delta,q_{v_0}, p \rangle$ 
has state set $Q$ and transition function $\delta$ as below. 
The priority function $p$ has $p(q_v) = 0$ for the states $q_v$ with $v\in V$, and $p(q) = 1$ for all other states.
{ \small
\begin{align*}
Q = & \{q_v\! \mid\! v\in V\} \ \bg{ \cup \ {q_{-}} } \cup \{q_r \mid r\in\SigR(\graph)\} \cup  \{q_{r^*}\! \mid\! r\in \SigR(\graph) \} \cup \{ q_{r^*v}\! \mid\! v \in V, r \in \SigR(\graph)\}
\end{align*}

{%
  \setlength{\abovedisplayskip}{0pt}%
  \setlength{\belowdisplayskip}{2pt}%
  \setlength{\abovedisplayshortskip}{0pt}%
  \setlength{\belowdisplayshortskip}{2pt}%

  \begin{align*}
    \delta(q_v, w) &= \begin{cases}
      \mathlarger{\bigwedge\limits_{1 \leq i \leq k} \Bigg (} 
      \bg{(i,q_-) \ \lor \ } \bigg ( (i,q_{r^*}) \land (i, q_{v}) \bigg )
      \ \lor \bigvee\limits_{ r(v,v') \in E} \\
      \qquad\qquad 
      \bigg ( ((i, q_r) \land (i, q_{v'})) \lor          
      \big ((i, q_{r^*}) \land (i, q_{r^*v'}) \big ) \bigg) 
      \Bigg ) & 
      \begin{aligned}    
        &\text{ if } v \in \semPath{A}_\graph  \\
        &\text{ for all } A \in w.\mathbf{A} 
      \end{aligned} 
      \\
      \bot & \text{ otherwise }
    \end{cases} 
  \end{align*}

  \begin{align*}
    \delta(q_{r^*v}, w) &= (0, q_v) \lor \bigvee\limits_{r(v,v') \in E} (0, q_{r^*v'}) 
    &
    \bg{\delta(q_{-},w)} &= \begin{cases} \top\!\!\!\!\! & \bg{\text{ if } - = w.R } \\
    \bot\!\!\!\!\! & \text{ otherwise }
    \end{cases} \\[1ex]
    \delta(q_{r},w) &= \begin{cases} \top\!\!\!\!\! & \text{ if } r = w.R 
    \\
    \bot\!\!\!\!\! & \text{ otherwise }
    \end{cases} 
    &
    \delta(q_{r^*},w) &= \begin{cases} \top\!\!\!\!\! & \text{ if } \roleStarSymbol{r}\! = w.R \\
    \bot\!\!\!\!\! & \text{ otherwise }
    \end{cases}
  \end{align*}
}
}

\bg{Intuitively, 
a successful run of $\automaton^{\graph, v_0}$
visits a node $w$ 
 of $\nodeLabelTreeWithoutC$ 
  in state $q_v$ if there is
a pre-match $\rho$ from $\nodeLabelTreeWithoutC$ into $v$ s.t. $\rho(w)= v$. 
Similarly, a visit to $w$ on $q_{r^*v}$ succeeds if there is
a pre-match $\rho$ from $\nodeLabelTreeWithoutC$ into $\graph$, but $\rho(w)= v$ is not enforced:  it suffices if the node $\rho(w)$ is reachable from $v$ along a path of $r$-edges.
The states $q_r$ and $q_{r^*}$ check for matching role labels, and $q_{-}$ checks if we have reached a dummy node in $\nodeLabelTreeWithoutC$. 
All transitions except those involving $q_{r^*v}$ require us to move to a new node in $\nodeLabelTreeWithoutC$, and 
the automaton halts whenever it reaches a node labelled $(-,\{\top\})$. 
This, together with the priority condition, ensures that all accepting runs on proper node-labelled $k$-trees are finite and witness the existence of a pre-match for the full concept corresponding to $\nodeLabelTreeWithoutC$. }
The following lemma establishes the correctness of the automaton. 

\begin{lemmarep} \label{lm:MatchingAutom} Given a pointed graph $(\graph,v_0)$, the $k$-2ATA 
$\automaton^{\graph, v_0}$ accepts a proper node-labelled \ourDL $k$-tree  
$\nodeLabelTreeWithoutC$ iff it has a pre-match into $(\graph,v_0)$. 
\end{lemmarep}

\begin{toappendix}
\begin{proof}
\bg{Let $\nodeLabelTreeWithoutC$ be a proper node-labelled \bg{\ourDL} $k$-tree with root $\varepsilon$,  and let $(\graph, v_0)$ be a pointed graph. \bg{
$\nodeLabelTreeWithoutC$ has an \emph{$r$-reachable} pre-match into $(\graph, v)$ if there exists some $v'$ in $\graph$ such that there exists a pre-match $\rho$ from $w$ in $\nodeLabelTreeWithoutC$ into $(\graph, v')$ s.t. $\rho(w)=v'$ and $v'$ can be reached from $v$ via a (possibly empty) path of $r$-edges.} 


We show the following:}

\begin{enumerate}
    \item Let $w\in\nodeLabelTreeWithoutC$ and let $w\neq \varepsilon$, and let $q_{-}\in Q$. There exists a partial $(w,q_{-})$-run  iff $w.R=-$. Moreover, this run will contain only finite paths. The analogous claims hold for partial $(w,q_{r})$-runs and partial $(w,q_{\roleStarOperation{r}})$-runs.
    \item Let $w\in\nodeLabelTreeWithoutC$ and let $q_{v}\in Q$. There exists an accepting partial $(w,q_v)$-run iff the tree $\nodeLabelTreeWithoutC^w$ rooted at $w$ has a pre-match into $(\graph, v)$.
    \item Let $w\in\nodeLabelTreeWithoutC$ and let $q_{r^*v}\in Q$. There exists an accepting partial $(w,q_{r^*v})$-run iff the tree $\nodeLabelTreeWithoutC^w$ rooted at $w$ has an $r$-reachable pre-match into $(\graph, v)$. 
    \item Given $w\in\nodeLabelTreeWithoutC$ and $q\in Q$, let $\nodeLabelTreeWithoutC_{\run'}$ be a partial $(w, q)$-run and $p$ an infinite path in $\nodeLabelTreeWithoutC_{\run'}$. If $q$ occurs infinitely often on $p$, it is of the form $q_{r^*v}$.
    \item If $\nodeLabelTreeWithoutC_{\run'}$ is an accepting partial $(w, q)$-run, all its paths are finite.
\end{enumerate}

Claim 1 follows directly from the definition of the transition function $\delta(q_{-}, w)$, $\delta(q_{r}, w)$ and $\delta(q_{\roleStarOperation{r}}, w)$ in $\automaton^{\graph,v_0}$.

Claim 2 follows from the transition function $\delta(q_v,w)$ inductively as follows. For the right-to-left direction, we show the contrapositive: If there is no pre-match from $\nodeLabelTreeWithoutC^w$ rooted at $w$ into $(\graph, v)$, then the concept labels in $w.\mathbf{A}$ do not match the label of $v$ in $\graph$, and the run fails. 

For the left-to-right direction, we assume there is a pre-match from $\nodeLabelTreeWithoutC^w$ into $(\graph, v)$. In that case the labels in $w$ match those in $v$. If $v$ has no outgoing edges, then the automaton can choose to either search for a child node $w\cdot i$ with $w\cdot i.R = -$, or with $w\cdot i. R = \roleStarSymbol{r}$ and the concept labels in $w\cdot i.\mathbf{A}$ match the label of  $v$. Since there exists a pre-match from $\nodeLabelTreeWithoutC^w$ into $(\graph, v)$, one of these cases must be successful. In the former case, there exists a partial $(w\cdot i,q_{-})$-run, which is finite by Claim 1, and therefore accepting. Similarly, if the latter check is successful, there exists an accepting partial $(w\cdot i,q_{\roleStarSymbol{r}})$-run, and since the labels in $w\cdot i.\mathbf{A}$ match the label of  $v$, there is a pre-match from $w \cdot i$ into $(\graph, v)$, and the same holds for any further successor nodes of $w\cdot i$ along the branch. Since we have defined $\nodeLabelTreeWithoutC$ as a proper node-labelled tree, every branch in $\nodeLabelTreeWithoutC^w$ will reach a node $w'=(-,\{\top\})$ in a finite number of steps. Thus, we repeat the previous check until we have found such a node $w'$, at which point we must have $w'.R=-$, and therefore the partial $(w',q_{-})$-run is accepting by Claim 1. It then follows by the definition of $\delta(w,q_v)$, that the partial $(w,q_v)$-run is accepting.  
In case $v$ does have some outgoing edge $r(v,v')\in E$,  the automaton has two options, one of which must give an accepting run. In the first option, it may choose to transition into $(w\cdot i, q_r)$ and $(w\cdot i, q_{v'})$. Since there is a pre-match from $\nodeLabelTreeWithoutC^w$ into $(\graph, v)$, we will have that $w\cdot i.R = r$ for some child node $w\cdot i$ of $w$ and there is a pre-match from the tree rooted at $w\cdot i$ into $(\graph, v')$. Therefore, the partial $(w\cdot i, q_r)$-run and the partial $(w\cdot i, q_{v'})$-run are both accepting, and therefore so is the partial $(w, q_v)$-run.
If the former option check fails, it may transition into $(w\cdot i, q_{r^*})$ and $(w\cdot i, q_{r^*v'})$. Since there must exist a pre-match from the tree rooted at $w\cdot i$ into $(\graph, v')$, by Claim 1 and Claim 3 both the partial $(w\cdot i, q_{r^*})$-run and the partial $(w\cdot i, q_{r^*v'})$-run will be accepting and therefore so is the partial $(w, q_v)$-run.

Claim 3 follows similarly, from the transition function $\delta(q_{r^*v},w)$. For the right-to-left direction, assume there exists an accepting partial $(w,q_{r^*v})$-run $\nodeLabelTreeWithoutC_{\run'}$. Assume for contradiction that there is no $r$-reachable pre-match from $\nodeLabelTreeWithoutC^w$ into $(\graph, v)$. Then, by construction of $\delta(q_{r^*v},w)$, the check $(0, q_v)$ will fail, so we must check $(w, q_{r^*v'})$ for some outgoing edge $r(v,v')\in E$. This check will be repeated until there are no more outgoing edges in $\graph$, in which case the run will fail, or it will loop infinitely often, in which case $(w,q_{r^*v})$ will occur infinitely often in this branch. Since such a state has $p(q_{r^*v})=1$ this also results in a non-accepting run.

For the left-to-right direction, we assume there exists an $r$-reachable pre-match from $\nodeLabelTreeWithoutC^w$ into $(\graph, v)$. In case $v$ has no outgoing edges in $\graph$, we move into $(w, q_v)$. Since there exists a pre-match from $\nodeLabelTreeWithoutC^w$ into some node $v'$ which is reachable via an $r$-path in $\graph$ from $v$, but $v$ has no outgoing $r$-edges, this $r$-path must be vacuous and the pre-match exists at $v$ itself, then by Claim 2 the partial run rooted at $(w, q_v)$ will be accepting. 
If there does exist some outgoing edge $r(v,v')\in E$, then we may also move into $(w, q_{r^*v'})$. By assumption, there also exists  an $r$-reachable pre-match from $\nodeLabelTreeWithoutC^w$ into $(\graph, v')$. So we may repeat this check along the $r$-path in $\graph$ until we have reached the node $v^i$ such that there exists a pre-match from $\nodeLabelTreeWithoutC^w$ into $v^i$, at which point the partial run rooted at $(w, q_{v^ i})$ will be accepting. Then the partial $(w, q_{r^*v^i})$-run will accept as well, and hence the partial $(w, q_{r^*v})$-run is accepting.

For Claim 4, assume we have a partial $(w, q)$-run and $p$ an infinite path in $\nodeLabelTreeWithoutC_{\run'}$ on which $q$ occurs infinitely often. By Claim 1, $q$ cannot be of the form $q_{-}$, $q_r$ or $q_{r^*}$. If $q$ is of the form $q_v$, then from the proof of Claim 2 and the properness of $\nodeLabelTreeWithoutC$, it follows that we will always reach a state $q_{-}$ in a finite number of steps. Thus $q_v$ also cannot occur infinitely often. Therefore any $q$ occurring infinitely often on $p$ must be of the form $q_{r^*v}$. 

Claim 5 follows as a corollary from Claim 4 and the priority function of $\automaton^{\graph,v_0}$.

It then follows from the acceptance conditions of $k$-2ATAs that $\nodeLabelTreeWithoutC_\run$ is an accepting run iff there exists a pre-match from $\nodeLabelTreeWithoutC$ into $(\graph, v_0)$. 

\end{proof}
\end{toappendix}


 $\overline{\automaton}$ denotes the  \emph{complement automaton} of a $k$-2ATA $\automaton$.   $\automaton_{\mathscr{A}_n}$ denotes the \emph{intersection automaton} over a finite set of $k$-2ATAs, $\mathscr{A}_n = \{\automaton_1, \dots \automaton_n\}$. 
Both set operations  are trivial for 2ATAs (cf. \cite{10.1007/BFb0055090}).
%
The complement automaton
 $\overline{\automaton^{\graph, v_0}}$ of $\automaton^{\graph, v_0}$ is a $k$-2ATA that accepts exactly those proper node-labelled $k$-trees that do not have a pre-match into $(\graph,v_0)$. 
 Hence, for a fitting instance $\fittingNoCat$, we can build $\automaton^{\graph, v}$  for every $v \in \PEx$, and $\overline{\automaton^{\graph, n}}$ for every $n \in \NEx$. Then we build the intersection automaton over all of these, which accepts exactly those proper node-labelled $k$-trees over $\Sigma$ that represent fittings for $\Fmc$.

\begin{toappendix}
\paragraph{Complementing and intersecting 2ATAs.} Given a $k$-2ATA $\automaton^{\graph, v_0}$, its \emph{complement automaton} $\overline{\automaton^{\graph, v_0}}$ can be defined in the usual way, by dualising the transition function and changing the priority function. For more details on this construction, we refer interested readers to~\cite{DBLP:journals/tcs/MullerS87}.


From \Cref{lm:MatchingAutom} and the construction of the complement automaton, we directly get the following corollary.
\begin{corollary}  \label{crl:ComplementAutom}
Given a data graph $\graph = (V,E,\ell)$ and a 
node $v \in V$,  its complement automaton is a $k$-2ATA $\overline{\automaton^{\graph, v_0}}$ that
accepts exactly those proper node-labelled $k$-trees $\nodeLabelTreeWithoutC$ such that there is not a pre-match from $\nodeLabelTreeWithoutC$ into $v$.
\end{corollary}

Given a set of $k$-2ATAs, $\mathscr{A}_n = \{ \automaton_n, \dots, \automaton_{n}\}$ of size $n$, where for each $\automaton_i = \langle Q_i, \labelLang, \delta_i, q_i, p_i \rangle$ with $1\leq i \leq n$, we can define the intersection automaton as a $k$-2ATA $\automaton_{\mathscr{A}_n} = \langle  Q_{\mathscr{A}_n}, \labelLang, \delta_{\mathscr{A}_n}, q_{\mathscr{A}_n}, p_{\mathscr{A}_n}  \rangle$ as follows: 
\begin{align*}
    Q_{\mathscr{A}_n} &= \bigcup_{1\leq i \leq n} Q_i \ \cup \ \{ q_{\mathscr{A}_n} \} \\
    \delta_{\mathscr{A}_n}(q,w) & = 
    \begin{aligned}
    \begin{cases}        
     \mathlarger{\bigwedge\limits_{1 \leq i \leq n}} (0,q_i) & \text{ if } q = q_{\mathscr{A}_n} \\
     \delta_i(q,w) & \text{ if otherwise } q \in Q_i \text{ for some } i 
    \end{cases}
    \end{aligned} \\
    p(q) & = 
    \begin{aligned}
        \begin{cases}        
     0 & \text{ if } q = q_{\mathscr{A}_n} \\
     p_i(q) & \text{ if otherwise } q \in Q_i \text{ for some } i 
    \end{cases}
    \end{aligned} 
\end{align*}

We say that $\automaton_{\mathscr{A}_n}$ is the intersection over $\mathscr{A}_n$.
\begin{lemma} \label{lm:IntersectionAutom}
    Given a set of $k$-2ATAs, $\mathscr{A}_n = \{ \automaton_1,  \dots, \automaton_n\}$, let  $\automaton_{\mathscr{A}_n}$ be the intersection over $\mathscr{A}_n$, then we have that $w \in L(\automaton_{\mathscr{A}_n})$ iff $w \in \bigcap_{1 \leq i \leq n} L(\automaton_i)$.
\end{lemma}
\end{toappendix}
 To decide fitting existence, it remains to show that there is always a sufficiently large $k$ so that the intersection automaton accepts some fitting if it exists. 
\begin{lemmarep} \label{lm:branching}
    Let $\mathcal{F} = (\graph, \PEx, \NEx, \Sigma)$ be an instance with $\graph = (V,E,\ell)$. If there is 
    an \ourDL fitting for $\mathcal{F}$, then there is an \ourDL fitting $C$ such that $\mathit{deg}(\ourDLTree_C) = |V|$. 
\end{lemmarep}
\begin{proof}

Take any \ourDL fitting $C$ for $\mathcal{F}$. Let $\ourDLTree_C$ be the \ourDL tree with degree $k$ for $C$, where $k$ > $|V|$. Let $n$ be the root node of $\ourDLTree_C$. Let $u$ by any node in $\ourDLTree_C$. By construction we know that $u$ has exactly $k$ successors, $y_1, \dots, y_k$. For $1 \leq i \leq k$, let $\ourDLTree_C|^i_u$ be the subtrees rooted at $u$, obtained by deleting all successors $y_i, \dots, y_k$ and the subtrees below them from $u$. Let $S_i$ be the set of all $v \in V$ such that there is a pre-match $\rho$ from $\ourDLTree_C|^i_u$ to $\graph$ with $\rho(u) = v$. From this definition, we have $S_1 \supseteq S_2 \supseteq \dots$.
There must then exist some $j \leq |V|$, such that $S_j = S_{j+1} = \dots = S_k$. This follows from the fact that $k > |V|$ and hence we cannot decrease the sets $S_i$ more often than there are elements in $V$.

Let $\ourDLTree_D$ be the \ourDL description tree obtained from replacing $u$ in $\ourDLTree_C$ with $\ourDLTree_D|^j_u$ and $D$ the corresponding shape of $\ourDLTree_D$,
i.e.\ of the subtree we get by removing all successors $y_{j+1}, \dots, y_k$ of $u$ from $\ourDLTree_C$. We show that $D$ is a fitting for $\mathcal{F}$. It follows from our construction that $D$ fits all the positive examples, since we obtained $D$ by dropping subexpressions from $C$. It remains to show that $D$ fits none of the elements in $\NEx$, or formally we want to show that there does not exist a $v \in \NEx$ s.t.\ $v \in D^\graph$.
Let us assume to the contrary that there is some $v \in \NEx$ such that there is a pre-match $\rho$ from $\ourDLTree_D$ to $\graph$, where $\rho(n) = v$ 
(and therefore also $v \in {D}^\graph$ via \Cref{lm:preMatchCapturesEval}). Then it follows that we can construct a pre-match $\rho'$ from $\ourDLTree_C$ to $\graph$, with $\rho'(n) = v$. To see why, consider that the only difference between $\ourDLTree_D$ and $\ourDLTree_C$ lies in the successors $y_{j+i}, \dots, y_m$ that we removed in $\ourDLTree_D$. Yet we know that $S_j = S_{j+1} = \dots = S_k$. In other words, we know that for all elements in $S_j$, there must be pre-matches to extend them to $\ourDLTree_C$.  This, however, contradicts our assumption that $C$ was a fitting for $\mathcal{F}$.

Note that this construction only reduces the branching of \emph{one} node to at most $j$, where $j \leq |V|$.
By applying this construction for any node in $\ourDLTree_C$, we can construct an \ourDL description tree with degree  $k'=|V|$ such that its corresponding shape is a fitting for $\mathcal{F}$.    \qed
\end{proof}
%
%
This completes the construction of the $k$-2ATA that accepts encodings of \ourDL fittings, if they exist. 
%
\begin{theoremrep} \label{thm:fittingExistenceIntersection}
Given a fitting instance \fittingNoCat, 
let $k$ be the number of nodes in $\graph$, and  let  $\automaton_{\fit}$  be the $k$-2ATA obtained by intersecting 
$\automaton^{\graph, v}$  for every $v \in \PEx$ and 
$\overline{\automaton^{\graph,n}}$ for each $n \in \NEx$. 
Then $\automaton_{\fit}$ accepts a \ayt{proper} node-labelled $k$-tree $\nodeLabelTree$ over the signature $\Sigma$ 
if and only if
$C$ is an \ourDL fitting for $\mathcal{F}$.
\end{theoremrep}

\begin{proof} 

We show the claim by considering each direction separately. \\
\noindent
For the left-to-right direction, we assume that $\nodeLabelTree$ is accepted by $\automaton_\fit$ and let $C$ be its corresponding \ourDL shape. We know from \Cref{lm:MatchingAutom} that each $\automaton^{\graph, v}$ accepts a proper node-labelled \ourDL $k$-tree  when there exists a pre-match from this tree into $v \in \PEx$. Furthermore, by \Cref{crl:ComplementAutom}, we know that each $\overline{\automaton^{\graph,n}}$ only accepts a proper node-labelled \ourDL $k$-tree if it has no pre-match into any $n \in \NEx$. From  \Cref{lm:IntersectionAutom}, we know that $\nodeLabelTree$ is accepted by every automaton in th set of $k$-2ATAs $\mathscr{A}_m$ (where $m=|\PEx\cup\NEx|\leq |V|$) consisting of all $\automaton^{\graph, v}$ for each $v\in\PEx$ and $\overline{\automaton^{\graph,n}}$ for all $n\in\NEx$, so it is also accepted by the intersection automaton $\automaton_\Fmc$ over $\mathscr{A}_m$. Together with  \Cref{lm:preMatchCapturesEval} it follows that for every $v \in \PEx$, we have $v \in \semPath{C}_\graph$ and for every $n \in \NEx$,  we have $n \not \in \semPath{C}_\graph$. Combined with the observation that $\Sigma(C) \subseteq \Sigma$ by our initial assumption, it follows that $C$ is a fitting for $\mathcal{F}$.   \\

For the right-to-left direction, we assume that there exists some \ourDL shape $C$ that is a fitting for \fittinginst. From \Cref{lm:branching} it follows that there must exist some \ourDL fitting $D$ such that  $\nodeLabelTreeOther$ has degree at most $|V|=k$. We know that for every $v \in \PEx$ (resp. every $n \in \NEx)$, it holds that $v \in \semPath{D}_\graph$ (resp. $n \not \in \semPath{D}_\graph$). By  \Cref{lm:preMatchCapturesEval}, it follows that $\nodeLabelTreeWithoutC_D$ has a pre-match into each $v \in \PEx$ (resp. has no pre-match into any $n \in \NEx$). Combining this observation with \Cref{lm:MatchingAutom} and its corollary, we know that $\nodeLabelTreeOther$ is accepted by every automaton in $\mathscr{A}_m$, if we set $m = |\PEx\cup\NEx|\leq |V|$. From \Cref{lm:IntersectionAutom} it then follows that $\nodeLabelTreeOther$ is accepted by the intersection automaton $\automaton_\fit$ over $\mathscr{A}_m$. \qed
\end{proof}

Finally, through a non-emptiness check on the $k$-2ATA $\automaton_\Fmc$, which can be done in \mbox{\textsc{ExpTime}}, we obtain the \mbox{\textsc{ExpTime}} upper bound on the \ourDL fitting existence problem for instances with an empty shape catalogue. 
	From the complexity results for the three semantics in \Cref{thm:LFPTime} and \Cref{lemma:FittingEmptyCat}, it follows that even with non-empty catalogues, the \ourDL fitting existence problem can be solved in \textsc{ExpTime} under WFS, STS and SUS. This and \Cref{prop:FittingExpHard} then imply \textsc{ExpTime}-completeness of the \ourDL fitting existence problem for under these three semantics.


\begin{theorem}
\label{thm:FittingExCompleteness}
 The $\ourDL$\!\! fitting existence problem 
under WFS, STS and SUS is \textsc{ExpTime}-complete. The hardness holds even 
for empty shape catalogues.
\end{theorem}

\begin{proof}
The hardness was shown in~\Cref{prop:FittingExpHard}.
Using \Cref{thm:fittingExistenceIntersection}, we can decide existence of an \ourDL fitting for $\Fmc=(\graph, \PEx,\NEx,\Sigma)$ by relying on a non-emptiness check on $\mathfrak{A}_\fit$. Emptiness of 2ATAs is decidable in single exponential time in the combined size of the state set and the parity condition \cite{10.1007/BFb0055090}; both are polynomially bounded in $\fit$, so the test is feasible in time exponential in the size of $\fit$. 
By Lemma \ref{lemma:FittingEmptyCat} and the complexity results for computing and testing assignments under the three semantics in \Cref{thm:LFPTime}, then the fitting existence problem can be solved in \textsc{ExpTime}  under all three semantics and for any fitting instance. Given $\Fmc'=(\Gmc^\alpha,\PEx,\NEx,\Sigma)$, for WFS we let $\Gmc^\alpha=\Gmc^\textsf{wf}$ and check in \textsc{ExpTime} if an \ourDL fitting exists for $\Fmc'$. For SUS and STS, we list all the possible $\alpha\in\Sem(\graph^\alpha,\emptyset)^\mathsf{sup}$ (resp. $\alpha\in\Sem(\graph^\alpha,\emptyset)^\mathsf{st}$) and for each of them, we decide if a fitting exists over $\Fmc'$. Since there are exponentially many such $\alpha$, this can be done in \textsc{ExpTime}. Then by \Cref{lemma:FittingEmptyCat}, for any $\Sem$, if a fitting exists for $\Fmc'$, this is also a fitting for $\fittinginst$ Thus the \ourDL fitting existence problem is decidable in \textsc{ExpTime} under the three semantics. \qed
\end{proof}
\section{Computation of Most Specific Fittings in \ourDL
}


Before presenting our computation algorithms for  most specific fittings, we note that
a fitting problem that has a solution need not have a most specific one. There may be infinitely many, increasingly specific solutions, as illustrated on this 
example, \bg{which also appears in}  \cite{BKM-IJCAI-99,cate2025extremalfittingproblemsconjunctive}. 
This holds even if the node set of $\graph$ and $\PEx$ are singletons and $\NEx = \emptyset$ and is unaffected by the presence of $r^-$ or $r^*$. 
%
\begin{example}
\label{ex:NoMSF}
Consider $\mathcal{F}=(\graph, \{v\},\emptyset)$ with $v$ the only node in $\graph$: 
\smash[c]{ \raisebox{-5pt}{
    \begin{tikzpicture}[scale=0.7,->, >=Stealth, node distance=2cm, every node/.style = {font=\normalsize}]
        \node  (a) at (0,0) {\nodepositive };
        \path (a) edge [loop right] node[right] {$r$} (a);
    \end{tikzpicture}
}}.
%
%
%
     Then $C= \exists \roleStarOperation{r}\!.\top$ and $D=\exists r.\top$ are both \ourDL fittings for $\mathcal{F}$, with $D\subseteq C$, thus $D$ is more specific than $C$. However, we also get the fitting $D'=\exists r.(\top \land \exists r.\top)$, with $D'\subseteq D$. In fact, any fitting for $\Fmc$ can be subsumed by a more specific one, hence there is no (\ELI or \ourDL) fitting that is an MSF. 
\end{example}


Even when MSFs exist for both \ELI and \ourDL, they need not coincide: 

\begin{example}
\label{ex:ELI*morespecific}
Consider a fitting instance $\mathcal{F}=(\graph, \PEx,\emptyset)$ with $\PEx$ marked green:
\vspace{-3mm}
\begin{center}
    \begin{tikzpicture}[scale=0.7,->, >=Stealth, node distance=2cm, every node/.style = {font=\normalsize}]

\node[label={[label distance=-1.5mm]90:A}]  (a) at (8,0) {\nodepositive};
\node[label={[label distance=-1.5mm]90:B}]  (a1) at (10,0) { \nodebullet};

\node  (b) at (0,0) {\nodepositive};
\node  (b1) at (-2,0) { \nodebullet};
\node  (b2) at (2,0) { \nodebullet};
\node  (b21)[label={[label distance=-1.5mm]90:A}] at (4,0) { \nodebullet}; 
\node  (b212)[label={[label distance=-1.5mm]90:B}] at (6,0) { \nodebullet};

\path (a) edge node[above] {$s$} (a1);
\path (b) edge node[above] {$s$} (b1);
\path (b) edge node[above] {$r$} (b2);
\path (b2) edge node[above] {$r$} (b21);
\path (b21) edge node[above] {$s$} (b212);
\end{tikzpicture}
\end{center}
\vspace{-4mm}
The \ELI MSF is the shape $C_1 = \exists s.\top$. However, in \ourDL we can express the shape $C_2 = \exists s.\top \land \exists \roleStarOperation{r}\!.(A \land \exists s. B)$, with $C_2\subseteq C_1$. In fact, $C_2$ is the \ourDL MSF.
\end{example}

\bente{removed reference ot Proposition 4 (appendix)}
\bg{The \ourDL and \ELI MSF need not coincide, and in fact, unlike for arbitrary fittings, the existence of an \ELI MSF does not imply the existence of an \ourDL MSF,} as illustrated by Example \ref{ex:runningExMSF}. We will revisit this later in the section. 

\begin{example}
\label{ex:runningExMSF}
Consider $\mathcal{F}=(\mathcal{G}, \PEx,\emptyset)$ with $\PEx$ marked green:
\vspace{-3mm}
\begin{center}
                \begin{tikzpicture}[scale=0.7,->, >=Stealth, node distance=2cm, every node/.style = {font=\normalsize}]

        \node  (v1) at (-2.5,0) {$\graph\textnormal{:}$};

        \node [label={[label distance=-2mm]120:A}] (a) at (-1,0) {\nodepositive};
        \node [label={[label distance=-2mm]30:B}] (b) at (1,0) {\nodebullet};

        \node  (v2) at (3.5,0) { };
        
        \node [label={[label distance=-2mm]120:A}] (a2) at (5,0) {\nodepositive};
        \node [label={[label distance=-2mm]30:B}] (b2) at (7,0) {\nodebullet};

        \path (a) edge [bend left] node[above] {$s$} (b);
        \path (b) edge [bend left] node[above] {$r$} (a);

        \path (a2) edge [bend left] node[above] {$r$} (b2);
        \path (b2) edge [bend left] node[above] {$s$} (a2);
        \end{tikzpicture}
    \end{center}
    \vspace{-3mm}
    There exists an \ELI MSF for \Fmc, namely the shape $A$. But there is no \ourDL MSF: while there are many \ourDL fittings, none of these are most specific. For example, the fitting $A \land \exists \roleStarOperation{r}\!.(\exists \roleStarOperation{s}\!. B)$ is less specific than $A \land \exists \roleStarOperation{r}\!.(\exists \roleStarOperation{s}. (B \land \exists \roleStarOperation{r}\!. (\exists \roleStarOperation{s}\!. A)))$. By traversing the cycle, any \ourDL shape can be expanded into a more specific one. 
\end{example}


In the remainder of this section, we provide a computation algorithm for the \ourDL MSF under WFS, STS and SUS, if it exists.
Similarly to Section~\ref{sec:FittingExistence}, we first present an algorithm for computing the MSF in the case where $\cat = \emptyset$, which is thus independent of the semantics. We then show how this can again be lifted to arbitrary catalogues using adorned graphs, thus yielding an algorithm and complexity bounds for WFS, STS and SUS, analogously to \Cref{sec:FittingExistence}.

To compute MSFs, \bg{we do not use tree automata, but instead rely on the notion of graph products, similarly to previous literature on MSFs
\cite{cate2025extremalfittingproblemsconjunctive,DBLP:conf/ijcai/FunkJLPW19}.
This more constructive approach allows us to not only decide whether an MSF exists, but also compute a representation of it.}
However, to preserve $\ourDL$ shapes when doing this operation,  
we need to keep track of $r^*$-paths. We do this using $\roleStarSymbol{r}$ symbols, 
as we did 
for description trees.
In the sequel, we call a graph \emph{starred} if its signature has $\roleStarSymbol{r}$ symbols from \exSigStar, and extend the interpretation function to take into account these 
symbols, such that $\semPath{r \cup \roleStarSymbol{r}}_\graph$ and $\semPath{(r \cup \roleStarSymbol{r})^*}_\graph$ are defined analogously to \Cref{tab:seme2}. 
Graphs that witness all (non-trivial) $\roleStarOperation{r}$-paths by explicit $\roleStarSymbol{r}$ symbols are called \emph{\starclosed}.

\begin{definition} Given a starred graph $\graph = (V,E,\ell)$, we say that $\graph$ is 
     \emph{\starclosed} if whenever there are nodes $v,v' \in V$  with  $v \neq v'$ and $(v, v')\in\semPath{(r \cup \roleStarSymbol{r})}_\graph$ and
      $(v, v')\not\in \semPath{r}_\Gmc$,
    we have $(v, v')\in \semPath{\roleStarSymbol{r}}_\Gmc$. 
 \end{definition}

We can add $\roleStarSymbol{r}$ symbols to a given graph and make it \starclosed. 
We also define a simple \enquote{clean-up} operation that removes from a starred graph some $ \roleStarSymbol{r}$ symbols that are not necessary to be \starclosed. 

\begin{definition}
Given $\graph$, let $\graph_\star$ be the starred graph obtained by adding an edge $\roleStarSymbol{r}(u_0,u_n)$ for all pairs $(u_0,u_n)$  of nodes with $u_n$ reachable from $u_0$ by a (possibly empty) sequence of $r$-edges 
$r(u_0,u_1),\cdots,r(u_{n-1},u_n)$. 
The \emph{redundant star removal} (RSR) of a given starred graph $\graph$, denoted $\rsr{\graph}$, is the result of removing from $\graph$
all edges $\roleStarSymbol{r}(u,u)$, and 
$\roleStarSymbol{r}(u,u')$ s.t. $r(u,u')$ is also in $\graph$.
\end{definition}

The goal of the $\graph_\star$ construction is to make any graph \starclosed. 
Note that $\graph_\star$ contains \enquote{loops} $\roleStarSymbol{r}(u,u)$ for all nodes $u$ and roles $r$, but its RSR does not.

\begin{lemma}
$\graph_\star$ is \starclosed for every graph \graph. If some \graph is \starclosed, so is $\rsr{\graph}$. 
\end{lemma}

Given a graph $\graph=(V, E, \ell)$  
and nodes $\vec{v}=( v_1, \ldots v_n)$ in $V$,  we let $\Pi(\graph,\vec{v})$ be the product of the pointed graphs $(\graph,v_i)$, 
defined as usual.
 We may call $\vec{v}$ the distinguished element of $\Pi (\graph,\vec{v})$. 

We define our \emph{star product}, which  preserves $ r^*$-paths. In a nutshell, we make the input \starclosed, construct its regular product, and remove redundant stars; we restrict it to the given signature as usual. 
The star product is a pointed starred graph with distinguished node $\vec{v}$. Importantly, by construction, it is \starclosed.

\begin{definition}[star product]
Given $\graph$, nodes $\vec{v}=( v_1, \ldots v_n)$ in $\graph$,  and a signature $\Sigma$, we denote by 
$\Pi_\star(\graph, \vec{v},\Sigma)$
the pointed graph $\rsr{\Pi(\graph_\star,\vec{v})}|_{\exSigStar}$. 
If $\Sigma = \Sigma(\graph)$, we may omit it and write simply 
$\Pi_\star(\graph, \vec{v})$.
\end{definition}



\begin{lemmarep}
$\Pi_\star(\graph, \vec{v},\Sigma)$ is \starclosed for every  $\graph$, $\vec{v}$ and $\Sigma$.
\end{lemmarep}

\begin{toappendix}
\begin{proofsketch}
Let $\Pi_\star(\graph, \vec{v},\Sigma)$ be the star product graph of some graph $\graph$ with distinguished node $\vec{v}$ and signature $\Sigma$. For any node $v$ in $\Pi_\star(\graph, \vec{v},\Sigma)$ \bg{(which could be $\vec{v}$ or any other node), by definition $v$ appears in every pointed graph $(\graph, v_i)$ used to construct $\Pi_\star(\graph, \vec{v},\Sigma)$.} Furthermore, if there exists a node $v'$ in $\Pi_\star(\graph, \vec{v},\Sigma)$ s.t. $(v, v') \in \semPath{(r\cup \roleStarSymbol{r})^*}_\graph$ and $(v,v')\notin \semPath{r}_\graph$, then there also exists such a path between $v$ and $v'$ in each pointed graph $(\graph, v_i)$. There then also exists an $\roleStarSymbol{r}$-edge from $v$ to $v'$ in each $(\graph_\star, v_i)$, where $\graph_\star$ is the star-closed version of $\graph$, by definition. Since this edge exists in each such pointed graph and $\roleStarSymbol{r}\in\exSigStar$, it is preserved when we take their product.  Note that removing "redundant" $\roleStarSymbol{r}$-edges from this product still allows the graph to be \starclosed and results in the star product $\Pi_\star(\graph, \vec{v},\Sigma)$ as defined. Thus $\Pi_\star(\graph, \vec{v},\Sigma)$ is still \starclosed.  \qed
\end{proofsketch}
\end{toappendix}
\begin{example}
\label{ex:runningExMSF2}
Consider Example \ref{ex:runningExMSF}  again. 
We first extend $\graph$ to $\graph_\star$. Then we obtain $\Pi_\star(\graph, \vec{v})= \rsr{\Pi(\graph_\star,\vec{v})}$, with the node $\vec{v}$ marked green:
\vspace{-3mm}
\begin{center}
        \begin{tikzpicture}[scale=0.5,->, >=Stealth, node distance=1cm, every node/.style = {font=\normalsize}]

        \node  (v1) at (-6,0.5) {$\graph_\star\textnormal{:}$};

        \node [label={[label distance=-2mm]120:A}] (a) at (-3,0) {\color{ForestGreen}  \nodepositive};
        \node [label={[label distance=-2mm]30:B}] (b) at (-1,0) {\nodebullet};


        \node [label={[label distance=-2mm]120:A}] (a2) at (-3,-2.5) {\color{ForestGreen}  \nodepositive};
        \node [label={[label distance=-2mm]30:B}] (b2) at (-1,-2.5) {\nodebullet};

        \path (a) edge [loop left] node[left] {$\roleStarSymbol{r}\!,\roleStarSymbol{s}\!$} (a);
        \path (b) edge [loop right] node[below right] {$\roleStarSymbol{r}\!,\roleStarSymbol{s}$} (b);
        \path (a) edge [bend left=45] node[above] {$s,\roleStarSymbol{s}\!$} (b);
        \path (b) edge [bend left=45] node[above] {$r,\roleStarSymbol{r}\!$} (a);

        \path (a2) edge [loop left] node[above left] {$\roleStarSymbol{r}\!,\roleStarSymbol{s}\!$} (a2);
        \path (b2) edge [loop right] node[right] {$\roleStarSymbol{r}\!,\roleStarSymbol{s}\!$} (b2);
        \path (a2) edge [bend left=45] node[above] {$r,\roleStarSymbol{r}\!$} (b2);
        \path (b2) edge [bend left=45] node[above] {$s,\roleStarSymbol{r}\!$} (a2);
        
\end{tikzpicture} \quad
\begin{tikzpicture}[scale=0.5,->, >=Stealth, node distance=1cm, every node/.style = {font=\normalsize}]        
        \node  (pr) at (-2.5,-2) {$\Pi_\star(\graph,\vec{v})\textnormal{:}$};

        \node  (a4) at (-1.5,-3.5) {\nodebullet};
        \node [label={[label distance=-1mm]90:A}]  (ab4) at (2.5,-2.5) {\color{ForestGreen}  \nodepositive};
        \node  (b4) at (6.5,-3.5) {\nodebullet};
        \node  [label={[label distance=-2mm]270:B}] (ba4) at (2.5,-4.5) {\nodebullet};

        \path (a4) edge [bend left=45] node[above] {$\roleStarSymbol{r}\!$} (ab4);
        \path (ab4) edge [bend right]  node[below] {$\roleStarSymbol{s}\!$} (a4);

        \path (a4) edge [bend right=45] node[below] {$\roleStarSymbol{r}\!$} (ba4);
        \path (ba4) edge [bend left] node[above] {$\roleStarSymbol{s}\!$} (a4);

        \path (b4) edge [bend right=45] node[above] {$\roleStarSymbol{s}\!$} (ab4);
        \path (ab4) edge [bend left]  node[below] {$\roleStarSymbol{r}\!$} (b4);

        \path (b4) edge [bend left=45] node[below] {$\roleStarSymbol{s}\!$} (ba4);
        \path (ba4) edge [bend right]  node[above] {$\roleStarSymbol{r}\!$} (b4);

        \path (b4) edge [bend left=16] node[above] {$s$} (a4);
        \path (a4) edge [bend left=16]  node[below] {$r$} (b4);      
\end{tikzpicture}
\end{center}
\vspace{-5mm}
\end{example}

\paragraph{Simulations.}
Intuitively, the $\ourDL$ shapes that are satisfied in the star product are exactly those that are satisfied in all the pointed graphs of the product. To make this precise,  we define a notion of simulation that captures the semantics of \ourDL over possibly starred graphs. 

\begin{definition}
\label{def:simulation}
Given a signature $\Sigma$, an \ourDL $\Sigma$-simulation $\zeta$ between graphs $\Imc = (V_\Imc,E_\Imc,\ell_\Imc)$ 
 and $\Gmc=(V_\Gmc,E_\Gmc,\ell_\Gmc)$  
 is a relation $\zeta \subseteq V_\Imc \times V_\Gmc $ that satisfies:

\begin{description}[leftmargin=!,labelwidth=\widthof{\bfseries AtomR},topsep=2pt,partopsep=0pt,itemsep=0pt, parsep=2pt]
\item[AtomR] If $(d,e) \in \zeta$ and $d \in \semPath{A}_\Imc$, then $e \in \semPath{A}_\Gmc$.
\item[Forth] For every $r \in \exSig$, if $(d,e) \in \zeta$ and $(d,d') \in \semPath{r}_\Imc$, then there is $(e,e') \in \semPath{r}_\Gmc$ with $(d',e') \in \zeta$.
\item[Forth*] For every $r \in \exSig$, if $(d,e) \in \zeta$ and $(d,d') \in 
{\semPath{(r\cup \roleStarSymbol{r})^{*}\!}}_\Imc$, then there is $(e,e') \in {\semPath{(r\cup \roleStarSymbol{r})^{*}}}_\Gmc$ with $(d',e') \in \zeta$.
\end{description}
For nodes $d \in V_\Imc$ and $e \in V_\Gmc$, 
we write $(\Imc,d) \preceqSig (\Gmc,e)$ if there is a $\Sigma$-simulation $\zeta$ where $(d,e) \in \zeta$.
If $\Sigma$ is $\Sigma(\Imc)$, then we may instead just write  $(\Imc,d) \preceq (\Gmc,e)$.
\end{definition}

$\ourDL$ simulations  capture the semantics of \ourDL shapes in the usual way.
\vspace{-1mm}
\begin{lemma}
\label{lemma:SimSatisfaction}
The following hold for all pointed graphs and signatures $\Sigma$: 
\begin{enumerate}[topsep=0pt,partopsep=0pt,itemsep=0pt, parsep=0pt]
    \item If $(\Imc,d) \preceqSig (\Gmc,e)$, then for any \ourDL $\Sigma$-shape $C$, if
    $d \in \semPath{C}_\Imc$ then 
$e \in \semPath{C}_\Gmc$.
\item For every \ourDL shape $C$, if $d \in \semPath{C}_\Imc$, then 
$\ourDLTree_C\preceq(\Imc, d)$.
\item For all \ourDL shapes $C$ and $D$, we have   $C\subseteq D$ iff $\ourDLTree_D\preceq\ourDLTree_C$. 
\end{enumerate}
\end{lemma}
%

We characterise the correctness of the star product in terms of $\Sigma$-simulations. 

\begin{lemma} 
\label{lemma:ProdAllGraph}
     Given a graph $\graph$, a signature $\Sigma$ and nodes $\vec{v}$ in $\graph$,
 for every $(\Imc, d)$  we have that
$
(\Imc, d)\preceqSig(\graph, v_i)$ for every $v_i\in\vec{v}$ iff  $(\Imc, d)\preceqSig \Pi_\star(\graph, \vec{v}, \Sigma)$. 
\end{lemma}
The star product of the positive examples captures their commonalities, and the most specific shape it satisfies 
is our MSF candidate.
However, to be a fitting, we need to make sure that, on the one hand, it does not simulate any negative example, and on the other, that it can be represented as a (finite) shape. We  prove now that the latter holds exactly when it can be unravelled in a finite way.   

\paragraph{Unravellings.}
Unravelling is a \bg{commonly used}  technique for transforming an arbitrary graph into a tree-shaped one.  
Given a (possibly starred)  graph $\graph = (V,E,\ell)$ and a signature $\Sigma$, \emph{a $\exSigStar$-path  $p$  (of length $n$)} is a sequence $p =v_0 r_0 v_1 \dots v_{n-1} r_{n-1} v_{n}$, with $v_i \in V$, $r_i \in \exSigStar$ and for every $r_i $ with $1\leq i < n$ we have that $(v_i,v_{i+1}) \in \semPath{r_i}_\graph$.
We say that $p$ \emph{starts with $v_0$} and \emph{ends with $v_n$}, and we define $tail(p)=v_n$.  
We write $pRb$ to mean a $\exSigStar$-path that extends $p$ by following an edge labelled with $R$ to a node $b$.
We may \bg{simply speak of paths }when $\exSigStar$ is clear from context.

Let $v \in V$.
The \emph{$\Sigma$-unravelling of $v$ in $\graph$} is the graph $\Unr_{v,\Sigma} = (V_\Unr, E_\Unr,\ell_\Unr)$ s.t.:
\begin{enumerate*}[label=\arabic*., itemjoin=\quad]
\item $V_\Unr $ contains all $\exSigStar$-paths that start with $v$,  
\item for all $p$, $\ell_\Unr(p) = \ell(tail(p))$, and 
\item $(p,R,pRb) \in E_\Unr$ for every path $pRb \in V_\Unr$.
\end{enumerate*}

If $\Sigma$ is $\Sigma(\graph)$,
we may omit it and simply write $\Unr_v$.
If we add the requirement that  $V_\Unr $ contains only $\exSigStar$-paths of length at most $m$ for some \bg{$m\in\mathbb{N}$} s.t. $m \geq 0$, then we obtain an \emph{$m$-bounded $\Sigma$-unravelling}, or just $m$-$\Sigma$-unravelling, written $\Unr_{v,\Sigma}^m$. Note that if $m$ is at least the 
length of the longest path, then $\Unr^m_{v,\Sigma} = \Unr_{v,\Sigma}$. 
An $m$-unravelling $\Unr^m_v$ is \emph{complete} if for all $pRa\in\Unr_v$ with $p\in\Unr^m_v$ but $pRa\notin\Unr^m_v$, there is some $p'\in\Unr^m_v$ 
with  $R(p,p') \in E_{\Unr^m}$ and $(\Unr_v, pRa)\preceq (\Unr_v, p')$.

Similarly as we did for description trees, we might simply write $\Unr_v$ to denote the pointed graph $(\Unr_v, v)$ of an unravelling $\Unr_v$ and its root $v$.

The following lemma 
connects our notions of simulation and  unravelling. 

\begin{lemma}
\label{lemma:SimUnravel}
    Given two pointed graphs $(\Imc,d)$ and $(\Gmc, e)$, we have
    \begin{enumerate}[topsep=2pt,partopsep=0pt,itemsep=0pt, parsep=2pt]
        \item $(\Imc, d)\preceq(\Gmc, e)$ iff $ \Unr_d\preceq (\Gmc, e)$, 
        and
        \item If $\Imc$ is a tree, 
        then $(\Imc, d)\preceqSig(\Gmc, e)$ iff $ (\Imc, d)\preceqSig \Unr_e$ for every $\Sigma$.
    \end{enumerate} 
\end{lemma}


\begin{toappendix}
\paragraph{Pre-matches, homomorphisms and simulations.}
In addition to the pre-matches and \ourDL simulations introduced in the body of the paper, for some proofs we also need to use the more standard notions of homomorphisms and simulations as in the \EL setting.

First, we extend our notion of pre-matches from \Cref{def:prematch} so that we may use starred graphs on both sides, i.e. 
\bg{we let $r\in\exSigStar(\graph)$} and 
we change Condition 3 of \Cref{def:prematch} as follows:
\begin{align*}
    &\text{3. if } \roleStarSymbol{r}(v,v') \in E,& \text{ then } (\rho(v),\rho(v')) \in \semPath{\roleStarOperation{(r\cup\roleStarSymbol{r})}}_\graph \ .
\end{align*}
Then a \emph{homomorphism} is defined as a pre-match  
that does not use Condition 3. 
This corresponds to the regular definition of homomorphisms, but where $\roleStarSymbol{r}$-edges can be used and are treated the same as a regular (possibly inverse) $r$-edges.
{Note that homomorphisms have the following property, which ensures that a path in some unravelling is of at most the same length as its homomorphic image.}

\begin{quote}
{$(\ddagger)$ Given two pointed tree-shaped graphs $(\Imc, d), (\graph, e)$ such that there exists a homomorphism $h:(\Imc,d)\rightarrow(\graph,e)$, then for any path $p\in\Imc$, $length(p)\leq length(h(p))$.}
\end{quote}

Moreover, we define a $\ELI$ $\Sigma$-\emph{simulation} as a less restricted version of the \ourDL $\Sigma$-simulation introduced in \Cref{def:simulation}, which does not require the third condition \textbf{Forth*} \bg{(thus treating possible $\roleStarSymbol{r}$-edges the same as regular $r$-edges)}. We write $(\Imc, d)\ELIpreceq_\Sigma(\graph, e)$ to denote that there exists an \ELI $\Sigma$-simulation between $(\Imc,d)$ and $(\graph, e)$ {w.r.t. some signature $\Sigma$}. Like in the case of \ourDL $\Sigma$-simulations, when $\Sigma=\Sigma(\Imc)$ we may omit it and simply write $(\Imc, d)\ELIpreceq(\graph, e)$.
Note that if an \ELI $\Sigma$-simulation exists between two non-starred graphs, then we have an analogous version of Lemma \ref{lemma:SimSatisfaction} for $\ELI$ shapes. 

In the case of tree-shaped graphs, it is well-known that an \ELI simulation between two graphs exists iff there exists a homomorphism \bg{(cf. \cite{cate2025extremalfittingproblemsconjunctive})}. 
\end{toappendix}

Note that a starred graph unravels into a starred tree-shaped graph.
Unravellings of \starclosed graphs are generally not \starclosed.
However, if the unravelled graph is \starclosed, and if a node reaches another along an $r ^*$-path in the unravelling, it will also have a similar `copy' as an immediate $\roleStarSymbol{r}$-child.
This is useful, since it
allows us to find simulations by looking at bounded unravellings. 


\begin{lemmarep}
\label{lemma:StarClosedUnrBound}
For all $\Sigma$ and all $(\Imc,d)$ and $(\Gmc, e)$ such that $\graph$ is \starclosed, 
\begin{enumerate}[topsep=0pt,partopsep=0pt,itemsep=0pt, parsep=0pt]
            \item 
            \label{lm:StarClosed-item1}For every \ourDL shape $C$ with $\Sigma(C)\subseteq\Sigma$ s.t. $\ourDLTree_C \preceq \Umc_e$, we have $\ourDLTree_C \preceq \Umc_e^m$, where $m=\depth(C)$. 
            \item \label{lm:StarClosed-item2} If $\Imc$ and $\Gmc$ are finite, 
            then $ \Unr_d \preceqSig (\Gmc, e)$ iff $\Unr_d^m \preceqSig (\graph, e)$ for all $m$-unravellings $\Unr_d^m$ of $\Imc$ at $d$.
    \end{enumerate}
\end{lemmarep}

The proof of \Cref{lm:StarClosed-item1} relies on a correspondence between simulations and pre-matches for \starclosed graphs.
\Cref{lm:StarClosed-item2} follows directly from the proof of Lemma 5.5. in \cite{cate2025extremalfittingproblemsconjunctive}, but adapted to \ourDL $\Sigma$-simulations. 


\begin{proof}
    For both items of \Cref{lemma:StarClosedUnrBound}, we rely on the following claim:

    \begin{quote}
        ($\dagger$) If $\Gmc$ is a \starclosed graph with some node $v$, then for every \ourDL shape $C$, there exists a pre-match $\rho$ from $\ourDLTree_C$ into $\Umc_v$, iff there exists a regular homomorphism $h:\ourDLTree_C\rightarrow\Umc_v$.
    \end{quote}
The right-to-left direction of this claim follows immediately, since a homomorphism in our setting can be seen as a special case of pre-match where for any $r$ or $\roleStarSymbol{r}$-edge in the description tree, there exists an $r$ or $\roleStarSymbol{r}$-edge in the graph.

The left-to-right direction of the claim follows from the following two properties of $\Unr_v$:
\begin{enumerate}
    \item For any two nodes $p, p'$ in $\Unr_v$, if $tail(p)=tail(p')$, then the subtrees below $p, p'$ are isomorphic to each other.
    \item By the fact that $\graph$ is \starclosed, it follows that for any node $p$ in $\Unr_v$, if there exists a node $q$ which can be reached from $p$ via an $(r\cup\roleStarSymbol{r})$-path, then there exists a node $q'$ which can be reached from $p$ via a single $\roleStarSymbol{r}$-edge s.t. $tail(q)=tail(q')$.
\end{enumerate}
Then let $\graph$ be a \starclosed graph with some node $v$ and let $C$ be s.t. $\ourDLTree_C=(V, E, \ell)$ has a pre-match $\rho$ into $\Unr_v$ such that
$\roleStarSymbol{r}(v, v')\in E$ and $(\rho(v), \rho(v'))\in \semPath{(r\cup\roleStarSymbol{r})^*}_{\Unr_v}$ and $(\rho(v), \rho(v'))\notin \semPath{(r\cup\roleStarSymbol{r})}_{\Unr_v}$. By the fact that $\graph$ is \starclosed, there exists a pre-match $\rho'$ from $\ourDLTree_C$ into $\Unr_v$ such that 
$\roleStarSymbol{r}(v, v')\in E$ and $(\rho'(v), \rho'(v'))\in \semPath{\roleStarSymbol{r}}_{\Unr_v}$,
with $\rho(v)=\rho'(v)$ and $tail(\rho(v'))=tail(\rho'(v'))$. Then the subtrees of $\rho(v')$ and $\rho'(v')$ are isomorphic, and therefore $\rho'$ is a homomorphism $\rho':\ourDLTree_C\rightarrow\Umc_v$.

To prove Lemma \ref{lemma:StarClosedUnrBound} Item 1, we assume that given a \starclosed graph $\graph$ with node $v$ and an \ourDL shape $C$ with $\Sigma(C)\subseteq\Sigma$, we have $\ourDLTree_C\preceq\Umc_v$. Then there also exists a pre-match from $\ourDLTree_C$ into $\Umc_v$. It then follows from $(\dagger)$ that there exists a homomorphism from $\ourDLTree_C$ to $\Umc_v$. Let $m=\depth(C)$ It follows from the property $(\ddagger)$ for homomorphisms, that then there also exists a homomorphism from $\ourDLTree_C$ to $\Umc^m_v$. By the other direction of $(\dagger)$, there is then also a pre-match from $\ourDLTree_C$ to $\Umc^m_v$ and hence also $\ourDLTree_C\preceq\Umc^m_v$.     \qed

\end{proof}

\paragraph{Algorithm for computing MSFs in \ourDL.}
In the following, $\vecV$ denotes an (arbitrary but fixed) enumeration of the positive nodes $\PEx$  of a fitting instance $\fittingNoCat$. 
Our algorithms for deciding existence of and computing an MSF build on the following proposition, similar to Prop. 5.14 in  \cite{cate2025extremalfittingproblemsconjunctive}.
%
\begin{propositionrep}
\label{prop:MSFEquivProd}
Let \fittingNoCat and $C$ be an \ourDL shape, then
    
    \begin{enumerate*}[(1)]
        \item $C$ is a MSF for $\Fmc$ \quad iff \quad {~}
        \item $C$ is a fitting for $\Fmc$ and $\FitProdStar\preceq \ourDLTree_C$.
    \end{enumerate*}
\end{propositionrep}

\begin{proof}

(1) $\Rightarrow$ (2). Assume that $C$ is an \ourDL MSF. Then $C$ fits all $v_i\in \PEx$, so by Lemma \ref{lemma:SimSatisfaction} Item 2, $\ourDLTree_C\preceq(\graph, v_i)$ for all $v_i\in \PEx$. By \Cref{lemma:ProdAllGraph}, then $\ourDLTree_C\preceq \FitProdStar$ and by Lemma \ref{lemma:SimUnravel} Item 2, $\ourDLTree_C\preceq \Unr_{\vecV}$. 

Let $m=\depth(C)$ and let $\Unr^m_{\vecV}$ be the $m$-bounded unravelling of $\FitProdStar$. Then by Lemma \ref{lemma:StarClosedUnrBound} Item 1, $\ourDLTree_C\preceq \Unr^m_{\vecV}$.
Suppose that $\Unr_\vecV\preceq(\graph, n_i)$ for some $n_i\in \NEx$. Then $\FitProdStar\preceq(\graph, n_i)$  by Lemma \ref{lemma:SimUnravel} Item 1 and, since $\ourDLTree_C\preceq\FitProdStar$, we have $\ourDLTree_C\preceq(\graph, n_i)$. This is a contradiction since $C$ is a fitting for $\Fmc$. So we must have $\Unr_\vecV\npreceq(\graph, n_i)$ for any $n_i\in \NEx$. Then by Lemma \ref{lemma:StarClosedUnrBound} Item 2,
$\Unr^m_{\vecV}\npreceq (\graph, n_i)$ for any $n_i\in \NEx$. Thus the corresponding shape $C_m$ of $\Unr^m_{\vecV}$ is a fitting for $\Fmc$. 
This same argument can then also be extended for any (potentially infinite) $m'$-unravelling s.t. $m'\geq m$. Hence the corresponding shape $C_{m'}$ of $\Unr^{m'}_\vecV$ is also a fitting for $\Fmc$. Since $C$ is an MSF, for any such $C_{m'}$ we have $C\subseteq C_{m'}$. 
By Lemma \ref{lemma:SimSatisfaction} Item 3, then $\Unr^{m'}_\vecV\preceq\ourDLTree_C$ for any bounded 
unravelling s.t. $m'\geq m$. 
In particular, by \Cref{lemma:StarClosedUnrBound} Item 2, this then also holds for the full unravelling $\Unr_\vecV$. So we have $\Unr_\vecV\preceq\ourDLTree_C$.   It then follows from Lemma \ref{lemma:SimUnravel} Item 1 that $\FitProdStar\preceq \ourDLTree_C$.
 
(2) $\Rightarrow$ (1). Assume  $C$ is an \ourDL fitting for $\Fmc$ and $\FitProdStar\preceq \ourDLTree_C$. Let $D$ be any \ourDL fitting for $\Fmc$. Then by Lemma \ref{lemma:SimSatisfaction} Item 2, $\ourDLTree_{D}\preceq(\graph, v_i)$ for all $v_i\in \PEx$. Therefore by \Cref{lemma:ProdAllGraph}, $\ourDLTree_{D}\preceq \FitProdStar$. Then we  have $\ourDLTree_{D}\preceq \ourDLTree_C$, and thus by Lemma \ref{lemma:SimSatisfaction} Item 3, $C\subseteq D$. Therefore $C$ is a MSF for $\Fmc$.   \qed
\end{proof}
 The proof follows from the fact that any bounded unravelling of $\FitProdStar$ for an instance $\Fmc$ is also a fitting for $\Fmc$ and is simulated by the \ourDL description tree of the MSF. The full unravelling, and therefore $\FitProdStar$ itself is then also simulated by the description tree of the MSF.
 
We can now decide MSF existence. We call $\FitProdStar$ a \emph{fitting product} (FP)  if 
$\FitProdStar \not \preceq (\graph,n)$ for all $n \in \NEx$. We also say $\FitProdStar$ \emph{fits} $\Fmc$.

%

\begin{theorem}
\label{theorem:ELI*msfExistence}
     Given \fittingNoCat, an \ourDL MSF exists 
iff $\FitProdStar$ fits $\Fmc$ and \bg{there exists} some complete $m$-unravelling \bg{of $\FitProdStar$}.
\end{theorem}

\begin{proof}
 First, assume that $C$ is an \ourDL MSF for $\Fmc$. By Proposition \ref{prop:MSFEquivProd}, $\FitProdStar\preceq\ourDLTree_C$ and, since $C$ fits all positive examples, $\ourDLTree_C\preceq\FitProdStar$. 
 Furthermore, $\ourDLTree_C\npreceq(\graph, n_i)$ for all $n_i\in\NEx$, therefore we have $\FitProdStar_\vecV\npreceq(\graph, n_i)$. Thus $\FitProdStar$ fits $\Fmc$.
 Let $m=\depth(C)$ and let $\Unr^m_{\vecV}$ be the $m$-unravelling of $\FitProdStar$. Then since  $\FitProdStar$ is \starclosed and by Lemma \ref{lemma:StarClosedUnrBound} Item 1, $\ourDLTree_C\preceq \Unr^m_{\vecV}$. Therefore $\FitProdStar\preceq\Unr^m_\vecV$, and by Lemma \ref{lemma:SimUnravel} Item 1, $\Unr_\vecV\preceq\Unr^m_\vecV$. Thus any path in $\Unr_\vecV$ can be simulated by some path in $\Unr^m_\vecV$, 
 in particular any path $pRa\in\Unr_\vecV \backslash \Unr^m_\vecV$ with $p\in\Unr^m_\vecV$ can be simulated by some path $p'$ with $R(p,p')\in E_{\Unr^m}$, 
 hence $(\Unr_\vecV, pRa)\preceq(\Unr_\vecV, p')$, and $\Unr^m_\vecV$ is complete.

    For the other direction, assume $\FitProdStar$ fits $\Fmc$ and a complete $\Unr^m_{\vecV}$ exists for some \co{$m\in\mathbb{N}$ s.t. $m \geq 0$}.
    Since $\Unr^m_\vecV$ is complete, given any path $pRa\in\Unr_\vecV \backslash \Unr^m_\vecV$ with $p\in\Unr^m_\vecV$,  there exists some $p'\in\Unr^m_v$ with $R(p,p')\in E_{\Unr^m}$ and $(\Unr_\vecV, pRa)\preceq (\Unr_\vecV, p')$. Therefore, any path in $\Unr_\vecV$ can be simulated by a path in $\Unr^m_\vecV$ and $\Unr_\vecV\preceq \Unr^m_\vecV$. 
    Then by Lemma \ref{lemma:SimUnravel} Item 1, $\FitProdStar\preceq\Unr^m_\vecV$. By Proposition \ref{prop:MSFEquivProd}, then the corresponding shape of $\Unr^m_\vecV$ is the \ourDL MSF for $\Fmc$.  \qed
 \end{proof}   

\bg{We illustrated in Example \ref{ex:runningExMSF} that even if an \ELI MSF exists, this does not imply that an \ourDL MSF exists as well, and now}
we can argue why.  \footnote{\bg{In the preliminary version 
\cite{Gortworst_Okulmus_Ortiz_Turhan_2026}, we incorrectly claimed the converse.}}

 \begin{example}
\label{ex:runningExMSF3}
 Consider again Example \ref{ex:runningExMSF}. 
 The unravelling $\Unr_{\vec{v}}$ of $\Pi_\star(\graph,\vec{v})$ is an infinite tree where each node has either one outgoing $\roleStarSymbol{r}$-edge and one $\roleStarSymbol{s}$-edge, or two outgoing $\roleStarSymbol{r}$ (resp. $\roleStarSymbol{s}$)-edges and one $r$ (resp. $s$)-edge. Note that there is no complete bounded unravelling of $\Pi_\star(\graph,\vec{v})$. By Theorem~\ref{theorem:ELI*msfExistence}, this implies that an \ourDL MSF for $\Fmc$ does not exist, as anticipated.
\end{example}

\SetAlgorithmName{Alg.}{alg}{list of algorithms}

\begin{toappendix}
\begin{figure}[t]
    \centering
\begin{algorithm}[H]
\small
\DontPrintSemicolon


\SetKwFunction{FMain}{$\mathtt{cbuDepth}$}
\SetKwProg{Fn}{function}{:}{}

\SetKw{Guess}{guess}
\SetKwFor{ForEach}{for each}{do}{end f.e.}
\KwIn{A pointed graph $(\graph, v_0)$, where $\graph = (V,E,\ell)$}
\KwOut{ A positive integer $m$ if a CBU exists, indicating the depth of the CBU
or -1 otherwise }

\Fn{\FMain{$(\graph, v_0)$}}{
    \texttt{count} $\coloneq 0$ \;
    $t\coloneq (-,-,v_0)$ \tcc*[r]{$t.s$, $t.p$, $t.o$ to access fields of triple.}  
    
    \ForEach{$v \in V$ and $r \in \SigR(\graph)$ s.t. $(t.o, v)\in \semPath{r}_\graph$ }{
        \If{$R^-\neq t.p$ or $(\graph, v)\npreceq (\graph, t.s)$}{   
            \Guess{some $v' \in V$ with $(t.o,v')\in \semPath{r}_\graph$ and $(\graph, v)\preceq(\graph, v')$} \;          
            $t\coloneq (t.o, r, v')$ \;
            \texttt{count} $\coloneq$ \texttt{count} $+1$\;
            \If{\texttt{count} $> |V|$}{
                \Return{-1} \;
            }
        }
    }
    
    \Return{\texttt{count}}  \; 
}

\caption{Compute depth of CBU}

\label{alg:existUnrm}
\end{algorithm}
\end{figure}
\end{toappendix}

\begin{figure}[t]
    \centering

\begin{minipage}[t]{0.52\textwidth}
\begin{algorithm}[H]
\small
\SetInd{0em}{0.7em}
\DontPrintSemicolon

\SetKwFunction{existCBU}{$\mathtt{cbuDepth}$}

\SetKwFunction{FMain}{$\mathtt{computeFP}$}
\SetKwProg{Fn}{function}{:}{}

\KwIn{$\mathcal{F} = (\graph,\PEx,\NEx,\Sigma)$}
\KwOut{FP for $\mathcal{F}$ if it  \\ 
\phantom{\bf Output: } exists or $\emptyset$ otherwise }

\Fn{\FMain{$\mathcal{F}$}}{
  $\vecV \coloneq $ $(v_0, \dots, v_\ell)$ with $v_i \in \PEx$ and $1 \leq i \leq \ell$ where $\ell = |\PEx|$  \; 
  $\Pi \coloneq \FitProdStar$ \;
\If{ \existCBU{$(\graph, \vecV)$ } $ = -1$ }{
        \Return{$\emptyset$}  \label{line:Alg1_return1}}
   \Else{     
   \ForEach{$n \in \NEx$}{
       \lIf{ $ \Pi \preceq (\graph ,n) $  }{
        \Return{$\emptyset$}   \label{line:Alg1_return2}
        }    }   }
   \Return{$\Pi$}  \;\label{line:Alg_1return3} 
}

\caption{Computing the fitting product}
\label{alg:existFitProd}
  \end{algorithm}
 \end{minipage}
\begin{minipage}[t]{0.47\textwidth}
\begin{algorithm}[H]
\small
\SetInd{0em}{0.7em}
\DontPrintSemicolon

\SetKwFunction{existFitProd}{$\mathtt{computeFP}$}

\SetKwFunction{existCBU}{$\mathtt{cbuDepth}$}

\SetKwFunction{FMain}{$\mathtt{computeMSF}$}
\SetKwProg{Fn}{function}{:}{}

\KwIn{$\mathcal{F} = (\graph,\PEx,\NEx,\Sigma)$}
\KwOut{\ourDL MSF for $\mathcal{F}$ if it \\ 
\phantom{\bf Output: } exists or $\emptyset$ otherwise }

\Fn{\FMain{$\mathcal{F}$}}{
  $\Pi \coloneq $  \existFitProd{$\mathcal{F}$}  \;
  $ m \coloneq $  \existCBU{$(\graph, \vecV)$}  \;
\If{ $\Pi = \emptyset$ }{
        \Return{$\emptyset$}  \label{line:Alg2_return1}}
  $\Unr_m \coloneq $  $m$-unravelling of $\Pi$  \;
\Return{corresponding shape of $\Unr_m$}  \;\label{line:Alg_2return3} 
}

\caption{Computing the \ourDL MSF}
\label{alg:existMSF}
  \end{algorithm}
 \end{minipage}

\end{figure}

We now present a computation algorithm for \ourDL MSFs in instances with an empty catalogue, see \Cref{alg:existFitProd} and \Cref{alg:existMSF}.The function $\mathtt{cbuDepth}$ in the algorithms takes as input some pointed graph and checks if we can obtain a complete bounded unravelling  (CBU) from this and returns its depth, or $-1$ if no CBU exists. This is done in polynomial time. 
\Cref{alg:existFitProd} then takes as input some fitting instance, constructs its star product $\FitProdStar$, which is done in exponential time, and returns it if it is a fitting product and a CBU exists for $\graph$, which is the case only if an MSF exists, by \Cref{theorem:ELI*msfExistence}. \Cref{alg:existMSF} constructs the complete bounded unravelling for $\FitProdStar$ and returns it as the MSF if it exists.

Note that the bounded unravelling of $\FitProdStar$ can be exponentially larger than $\FitProdStar$, hence the MSF may be double exponential in $\fit$, and exponential even when $\FitProdStar$ is polynomial. However, $\FitProdStar$ is a succinct  representation of the MSF when it exists. Thus, since the construction of $\FitProdStar$ is polynomial in $|\graph|$ if the number of positive examples is bounded, in this case we can decide MSF existence in polynomial time.

\begin{toappendix}
    \begin{lemma}
       Given a pointed graph $(\graph, v_0)$, 
       \Cref{alg:existUnrm} runs in time polynomial in the size of $\graph$, and it returns $\mathtt{count} \geq 0$ iff there exists some complete bounded unravelling of $(\graph, v_0)$.
    \end{lemma}

    \begin{proofsketch}
        {\Cref{alg:existUnrm} runs in time polynomial in $|V|$: it will at most go through every loop $|V|^3 \times \SigR(\graph)$ times and in each loop checks in polynomial time whether there exists a simulation between pointed graphs. It remains to show that a run of \Cref{alg:existUnrm} on some pointed graph $(\graph, v_0)$ is successful if and only if $(\graph, v_0)$ has some complete bounded unravelling. } 

    {If such a CBU $\Unr^m_{v_0}$ exists, then for every path $p\in\Unr^m_{v_0}$, and for every $pRv\in\Unr_{v_0}$, it must be the case that $v$ will be simulated by some $v'$ in $\Unr^m_{v_0}$, which is either a parent node of $v$, or $v$ can reach $v'$ via an inverse role. When running $\mathtt{cbuDepth}$ on $\graph$, the algorithm will therefore detect $(\graph, v)\preceq (graph, v')$ when it reaches this node $v$ and thus necessarily halt.  Furthermore, if the longest path in $\Unr^m_{v_0}$ is longer than $|V|$, it must be the case that some node $n$ in this path is visited twice (otherwise it is not possible for the path length to exceed the number of nodes in $V$). However, in this case the path contains a cycle and thus $n$ will be simulated by itself. Then this path can be cut off at the first appearance of $n$, which will be shorter than $|V|$. Therefore the algorithm will halt before the count reaches $|V|$ and return $\mathtt{count}$. }

    {For the other direction, assume the algorithm has a successful run on $(\graph, v_0)$, i.e. it returns $\mathtt{count}$. This means it halts on some $v$, such that $(\graph, v)$ is simulated by some $(\graph, v')$, with $v'$ a parent node, or it is reachable via an inverse role from $v'$. Let $pRv$ be the path ending in $v$. Then $pRv\in \Unr_{v_0}$, and $v'$ is in $p$. Let $\Unr^m_{v_0}$ be some $m$-unravelling s.t. $p\in \Unr^m_{v_0}$ and $p$ is the longest path in $\Unr^m_{v_0}$. Then for any $pRv\notin\Unr^m_{v_0}$, since the algorithm halts on $v$, this path is simulated by some $pRv'\in\Unr^m_{v_0}$. Hence,  $pRv'$ is s.t. $(\Unr_{v_0}, pRv)\preceq (\Unr_{v_0}, pRv')$ and $\Unr^m_{v_0}$ is a complete $m$-unravelling of $(\graph, v_0)$.}  \qed
    \end{proofsketch}
\end{toappendix}

\begin{theoremrep} \label{thm:MSFAlgoCorrectness} Given a fitting instance $\fittingNoCat$, it holds that:
\begin{enumerate}[topsep=0pt,partopsep=0pt,itemsep=0pt, parsep=0pt]
\item \Cref{alg:existFitProd} returns an FP iff an \ourDL MSF for $\mathcal{F}$ exists. It runs in \textsc{ExpTime} . 
\item \Cref{alg:existMSF} returns an \ourDL shape $C$ iff $C$ is an \ourDL MSF for $\mathcal{F}$. 
\item If $|\PEx| \leq k$ for a constant $k$, existence of an \ourDL MSF is  in \textsc{PTime}. 
\end{enumerate}
\end{theoremrep}

\begin{proof}

(1)  Let $V$ be the nodes of $\graph$. It is well-known that the standard product $\Pi(\graph,\vec{v})$ of pointed graphs can be constructed in time that is bounded by a polynomial in $|V|$ and an exponential in $|\vec{v}|$, but the construction only needs polynomial time if $|\vec{v}|$ is a constant.  Since our star product only adds pre- and post-processing steps that are polynomial in $|V|$, these bounds apply as well. $\mathtt{cbuDepth}$ runs in time polynomial to $|\FitProdStar|$, which in general may be of exponential size. Then, \Cref{alg:existFitProd} checks whether $\FitProdStar$ simulates into the pointed graph of any of the negative examples in $\NEx$, this also only needs polynomial time. It then concludes by outputting the fitting product in \textsc{ExpTime} if it exists, it follows directly from \Cref{theorem:ELI*msfExistence} that this is the case iff an \ourDL MSF exists.

(2) Follows directly from \Cref{theorem:ELI*msfExistence}. {Specifically, the right-to-left direction of the proof of \Cref{theorem:ELI*msfExistence} implies that we can return the corresponding shape of the CBU as the MSF.}

To show (3), note that $\mathtt{cbuDepth}$ runs in time polynomial to $|\FitProdStar|$ which is clearly polynomial if $|\PEx| \leq k$. Thus, in this case \Cref{alg:existMSF} runs in polynomial time. \Cref{alg:existFitProd} then returns, in polynomial time, the fitting product $\Pi$ as a representation of the MSF iff it exists. \qed
\end{proof}
\vspace{-2mm}



We can now extend this result to fitting instances that may have nonempty catalogues. To do so, we reduce to the setting with adorned graphs as in the proof of \Cref{thm:FittingExCompleteness} and use \Cref{lemma:FittingEmptyCat} to obtain \textsc{ExpTime}-upper bounds w.r.t. WFS, STS and SUS. Furthermore, by \Cref{thm:MSFAlgoCorrectness}  Item 3, if the number of positive examples is bounded by a constant, then under WFS, we can decide if a \ourDL MSF exists and compute it in polynomial time, since under these semantics we only need to iterate over the adorned graph $\graph^\mathsf{wf}$ once. This does not generally hold for STS or SUS, as there we might have \mo{exponentially} many adorned graphs.

\begin{theoremrep}
\label{thm:MSFExUpperBounds}
    Deciding the existence of an \ourDL MSF is in \textsc{ExpTime} under \ WFS, STS and SUS.     
    Under WFS, if additionally $|\PEx| \leq k$ for a constant $k$,  then MSF existence and computing a representation of it are  in polynomial time.
\end{theoremrep}

\begin{proof}
This follows directly from  Theorem \ref{thm:MSFAlgoCorrectness} for $\cat=\emptyset$. In the presence of a non-empty catalogue, we iterate over the adorned graphs as in the proof of Theorem \ref{thm:FittingExCompleteness}. 
The \textsc{ExpTime} upper bound then 
follows by \Cref{lemma:FittingEmptyCat} for the three semantics specified. If the number of positive examples is bounded by a constant, then under WFS, by Item 3 of \Cref{thm:MSFAlgoCorrectness}, deciding the existence of an \ourDL MSF can be done in polynomial time. Under WFS we only need one 
test over the polynomially computable adorned graph $\graph^\mathsf{wf}$ and by Alg.\ref{alg:existFitProd}, we obtain the fitting product as a representation of the \ourDL MSF.\qed 
\end{proof}


We stress that the polynomial upper bound of  \Cref{thm:MSFExUpperBounds} applies to many practically relevant cases on which the WFS coincides with the other semantics. Whenever the catalogue is positive  or non-recursive, $|\PEx| \leq k$ is a sufficient condition for tractability under any semantics. It is also sufficient under STS in the presence of recursion and negation if the catalogue admits a stratification.

\section{Conclusions and Future Work}

We provided an initial principled study on the \bg{formal foundations of the} problem of learning arbitrary and most specific fittings from sets of positive and negative examples on data graphs and SHACL shape catalogues. For the considered fragment of SHACL, which is equivalent to the DL \ourDL, we have shown that the fitting existence problem is \textsc{ExpTime}-complete under WFS, SUS, and STS for recursive SHACL. 
Furthermore, 
we gave an \textsc{ExpTime} algorithm to decide if an \ourDL MSF exists and, if it does, to compute it. Interestingly, under WFS and if the number of positive examples is bounded, we can decide the existence of an \ourDL MSF and return a representation of it even in \textsc{PTime}. 
We note that the only properties of the semantics that are critical to our results are those summarised in Theorem~\ref{thm:LFPTime}.Our results
hold for any semantics satisfying those conditions. In particular, our tractability results can be 
transferred to other SHACL settings 
with
tractable validation, e.g.\ for stratified shape catalogues.

As future research,  we would like to expand the fitting target language further towards full SHACL.
A natural step is to adopt counting over plain predicates. More advanced SHACL features like counting over paths, path equality, and closedness are exciting but also challenging and would require new techniques.
Additionally, approaches for learning \bg{nondeterministic finite automata} could be used for learning regular path expressions in SHACL. 
In this paper, we have focused on arbitrary and most specific fittings for SHACL. Other variants of fittings, such as most general, most succinct, and bounded fittings for SHACL remain unexplored, but are certainly interesting. 
We are particularly interested in whether our approach can be used to develop practical algorithms for finding shape expressions for existing RDF graphs. Several related approaches in the DL and CQ setting have been left untested or only tested on small tailored benchmarks, but in the setting of SHACL there is a huge amount of RDF data that could enable a much better practical evaluation.

\subsubsection*{Acknowledgments.}
This research was funded in whole or in part by the Austrian Science Fund (FWF) 10.55776/COE12 and 10.55776/PIN8884924.

\paragraph*{Declaration of use of Generative AI:}
  The author(s) have not employed any Generative AI tools.
%
%
%
%
\bibliographystyle{splncs04.bst}

\bibliography{bib} 

\end{document}